\documentclass[twoside,11pt,preprint]{article}

\usepackage{blindtext}

\usepackage{subcaption}
\usepackage{amsmath}
\usepackage{color}
\usepackage{enumitem}

\usepackage{jmlr2e}
\hypersetup{hidelinks}

\definecolor{NavyBlue}{rgb}{0.1,0.1,0.6}

\newcommand{\diff}{\mathrm{d}}

\newcommand{\R}{\mathbb{R}}

\newcommand{\bu}{\boldsymbol u}
\newcommand{\bI}{\boldsymbol I}
\newcommand{\bA}{\boldsymbol A}
\newcommand{\bC}{\boldsymbol C}
\newcommand{\bZ}{\boldsymbol Z}
\newcommand{\bP}{\boldsymbol P}
\newcommand{\bnu}{\boldsymbol \nu}
\newcommand{\bK}{\boldsymbol K}
\newcommand{\bq}{\boldsymbol q}
\newcommand{\bx}{\boldsymbol x}
\newcommand{\bX}{\boldsymbol X}
\newcommand{\bv}{\boldsymbol v}
\newcommand{\by}{\boldsymbol y}
\newcommand{\br}{\boldsymbol r}
\newcommand{\bw}{\boldsymbol w}
\newcommand{\bvarphi}{\boldsymbol \varphi}
\newcommand{\bz}{\boldsymbol z}
\newcommand{\bk}{\boldsymbol k}
\newcommand{\bmu}{\boldsymbol \mu}
\newcommand{\btheta}{\boldsymbol \theta}
\newcommand{\bM}{\boldsymbol M}
\newcommand{\bzero}{\mathbf{0}}

\renewcommand{\leq}{\leqslant}
\renewcommand{\geq}{\geqslant}

\newtheorem{thm}{Theorem}
\newtheorem{lem}{Lemma}
\newtheorem{coro}{Corollary}

\usepackage{lastpage}
\jmlrheading{24}{2023}{1-\pageref{LastPage}}{12/22; Revised
	4/23}{5/23}{22-1395}{Rapha\"el Berthier}

\ShortHeadings{Incremental~Learning in~Diagonal~Linear~Networks}{Berthier}
\firstpageno{1}

\begin{document}

\title{Incremental~Learning in~Diagonal~Linear~Networks}

\author{\name Raphaël Berthier \email raphael.berthier@epfl.ch \\
	\addr EPFL\\
	Lausanne, Switzerland}

\editor{Lorenzo Rosasco}

\maketitle

\begin{abstract}
	Diagonal linear networks (DLNs) are a toy simplification of artificial neural networks; they consist in a quadratic reparametrization of linear regression inducing a sparse implicit regularization. In this paper, we describe the trajectory of the gradient flow of DLNs in the limit of small initialization. We show that incremental learning is effectively performed in the limit: coordinates are successively activated, while the iterate is the minimizer of the loss constrained to have support on the active coordinates only. This shows that the sparse implicit regularization of DLNs decreases with time. This work is restricted to the underparametrized regime with anti-correlated features for technical reasons. 
\end{abstract}


\section{Introduction}
\label{sec:intro}

Artificial neural networks are the state of the art for many machine learning tasks~\citep{lecun2015deep}; however, we lack theoretical understanding of this success \citep{zhang2021understanding}. Indeed, the parametrization of neural networks induces a non-convex loss, and consequently it is challenging to analyze the optimization error of gradient descent methods. Moreover, neural networks can be successful even without any (explicit) regularizer; this challenges the statistical wisdom in overparametrized settings. 

Recent research suggests that these two problems are intertwined: through its non-convex parametrization, the gradient descent dynamics of neural networks induce an implicit regularization that controls the statistical performance \citep{bartlett2021deep}. However, this phenomenon is difficult to describe because it is a joint effect of the parametrization, the gradient descent dynamics and the initialization. 

As a consequence, theoretical research has focused on studying implicit regularization in toy simplifications of neural networks \citep{soudry2018implicit,gunasekar2017implicit,li2018algorithmic,chizat2020implicit,li2020towards}. We are interested in an extreme simplification, called \emph{diagonal linear networks} (DLNs) \citep{vaskevicius2019implicit,zhao2019implicit,woodworth2020kernel,haochen2021shape,li2021implicit,azulay2021implicit,pesme2021implicit,vivien2022label,nacson2022implicit,chou2021more}. In fact, it is only a linear regression where regressors $\theta_i$ are parametrized quadratically; specifically, in this paper, we parametrize $\theta_i = u_i^2/4$. We then perform a gradient descent in terms of $u_i$ and not $\theta_i$ (see Section \ref{sec:setting} for more details). This quadratic reparametrization is loosely argued to have an effect similar to the composition of two layers in a neural network. When started from a small initialization, DLNs were rigorously shown to enforce an implicit sparse regularization; in an overparametrized setting, DLNs converge to a sparse interpolator. 

Previous works have thus focused on describing the limit point of the dynamics of DLNs. Instead, in this work, we study the full trajectory of the continuous-time gradient flow dynamics. In the limit of small initialization, we show that \emph{incremental learning} (see, e.g., \cite{saxe2019mathematical,gidel2019implicit,gissin2019implicit,li2020towards}) is effectively performed: as time increases, coordinates are successively activated and the iterate is the minimizer of the loss constrained to have support on the active coordinates only. The main contribution of this paper is the description of the time-dependent set of active coordinates, and the rigorous proof of the convergence to a regressor whose sparsity depends on the stopping time. The take-home message is that DLNs enforce a sparse implicit regularization that decreases as the stopping time increases. 

As a corollary of our description of the dynamics, we obtain an asymptotic equivalent of the convergence time to the minimizer of the loss in the limit of small initialization. It is quite remarkable that such a precise estimate of the global convergence time can be obtained for a non-convex optimization problem.

For technical reasons, our work is restricted to the special case of anti-correlated feature; consequently, we study only underparametrized problems (see Section \ref{sec:assumptions}). However, we do not expect this restriction to be necessary for incremental learning to occur in DLNs. We leave the proof of this for future work.

The rest of this paper is organized as follows. In Section \ref{sec:main}, we present our setting, our assumptions and our results. In Section \ref{sec:related}, we articulate the related work in more detail. Sections \ref{sec:proof-thm-main} and \ref{sec:proof-coro-main} prove our results. 

\bigskip\noindent
\emph{Notations.} 
We use bold notations for vectors and matrices: for instance, if $\btheta \in \R^d$ is a multi-dimensional vector, we denote $\theta_i$ its coordinates. Similarly, if $\bM \in \R^{d \times d}$, we denote $M_{ij}$ its entries. If $I$ is a subset of $\{1,\dots,d\}$, we denote $I^{c}$ its complement and $|I|$ its cardinality. We denote $\btheta_I \in \R^{|I|}$ the subvector obtained from $\btheta \in \R^d$ by keeping only the coordinates indexed by $i \in I$. Similarly, if $I,J$ are subsets of $\{1,\dots,d\}$, we denote $\bM_{IJ}$ the submatrix obtained from $\bM$ by keeping only the rows indexed by $i \in I$ and columns indexed by $j \in J$.

If $\btheta, \bnu \in \R^d$, we write $\btheta \geq \bnu$ (resp.~$\btheta > \bnu$) to denote that for all $i \in \{1,\dots,d\}$, $\theta_i \geq \nu_i$ (resp.~$\theta_i > \nu_i$). In particular, $\btheta \geq \bzero$ (resp.~$\btheta >\bzero$) denotes that all coordinates of $\btheta$ are non-negative (resp.~positive). 

We use the notations $\langle . , . \rangle$ and $\Vert . \Vert$ to denote the Euclidean dot product and norm respectively.

\section{Main Results}
\label{sec:main}

In this section, we first introduce the parametrization of DLNs and the induced gradient flow dynamics (Section \ref{sec:setting}). Then, we state the assumption that features are anti-correlated and discuss the consequences (Section \ref{sec:assumptions}). Finally, we state our results (Section \ref{sec:statement}).

\subsection{Setting}\label{sec:setting} We perform a linear regression of $n$ output variables $y_1, \dots, y_n \in \R$ from $n$~corresponding input variables $\bx_1, \dots, \bx_n \in \R^d$. The traditional approach is to minimize the quadratic loss:
\begin{equation}
	\label{eq:optim-pb}
	\underset{\btheta \in \R^d}{\rm{min}.} \left\{ f(\btheta) = \frac{1}{2} \sum_{k=1}^{n} \left(y_k - \left\langle \btheta, \bx_k \right\rangle \right)^2 \right\} \, .
\end{equation}
Denote $\by = (y_1, \dots, y_n) \in \R^n$, $\bX \in \R^{n \times d}$ the matrix whose rows are $\bx_1,\dots,\bx_n \in \R^d$, and $\bX_1, \dots, \bX_d \in \R^n$ the columns of the matrix $\bX$, i.e., the features of the regression problem. The quadratic loss $f$ can be expressed as a function of the covariance $\bM = \bX^\top \bX \in \R^{d \times d}$ of the features and of the covariance $\br = \bX^\top \by$ between the features and the output:
\begin{equation*}
	f(\btheta) = \frac{1}{2} \Vert \by \Vert^2 - \langle \br, \btheta \rangle + \frac{1}{2} \langle \btheta,  \bM \btheta \rangle \, .
\end{equation*}
A strategy to minimize this convex function is to perform a gradient descent, i.e., an Euler discretization of the gradient flow 
\begin{equation}
	\frac{\diff \theta_i}{\diff t} = -  \frac{\diff f}{\diff \theta_i} = r_i - \sum_{j=1}^{d} M_{ij} \theta_j \, , \label{eq:GF-classical}
\end{equation}
where $\btheta = \btheta(t)$ is a function from $\R_{\geq 0}$ to $\R^d$. It is widely known that for any initial point, this gradient flow converges exponentially fast to a minimizer of $f$. 

In this paper, we are interested in the effect of reparametrizing 
\begin{equation*}
	\theta_i = \frac{1}{4} u_i^2 \, .
\end{equation*}
This reparametrization of linear regression is called a diagonal linear network (DLN). We perform the gradient flow in terms of $\bu \in \R^d$ instead of $\btheta \in \R^d$: $\frac{\diff u_i}{\diff t} = -  \frac{\diff f}{\diff u_i}$. Using that $\diff \theta_i = \frac{1}{2} u_i \diff u_i$, we compute the resulting equation in $\theta_i$:
\begin{align*}
	\frac{\diff \theta_i}{\diff t} = \frac{1}{2} u_i \frac{\diff u_i}{\diff t} = - \frac{1}{2} u_i \frac{\diff f}{\diff u_i} = -\frac{1}{4} u_i^2 \frac{\diff f}{\diff \theta_i} \, , 
\end{align*}
and thus 
\begin{align}
	\frac{\diff \theta_i}{\diff t} = - \theta_i \frac{\diff f}{\diff \theta_i} = \theta_i \left( r_i - \sum_{j=1}^{d} M_{ij} \theta_j \right) \, . \label{eq:GF-reparametrized}
\end{align}
Compare \eqref{eq:GF-reparametrized} with \eqref{eq:GF-classical}. The reparametrization has added a factor $\theta_i$ in the derivative of $\theta_i$. This implies that if $\theta_i$ is initialized at $0$, then it remains at $0$ in the DLN dynamics \eqref{eq:GF-reparametrized}. In particular, $\btheta = \bzero \in \R^d$ is a stable point of the dynamics. 

In this paper, we are interested in the DLN when initialized close to this stable point. More precisely, for $\varepsilon > 0$, define $\btheta^{(\varepsilon)} = \btheta^{(\varepsilon)}(t)$ as the solution of the DLN dynamics \eqref{eq:GF-reparametrized} initialized from $\btheta^{(\varepsilon)}(0) = (C_1 \varepsilon^{k_1}, \dots, C_d \varepsilon^{k_d})$, where $\bC = (C_1, \dots, C_d) > \bzero$ and  $\bk = (k_1, \dots, k_d) > \bzero$ are constants. 

\subsection{Assumptions} 
\label{sec:assumptions}
Before we get to a rigorous statement of our results, let us state our assumptions.
\begin{enumerate}[font={\bfseries},label={(A\arabic*)}]
	\item\label{it:ass-r} $\br = \bX^\top \by > \bzero$, i.e.,~the covariance $r_i = \langle \bX_i, \by \rangle$ between the output $\by$ and the feature $\bX_i$ is positive for all $i \in \{1, \dots, d\}$.
	
	\bigskip\noindent
	The reparametrization $\theta_i = \frac{1}{4} u_i^2 $ constrains the linear regression to have non-negative weights. In this situation, it is natural to preprocess the data by potentially changing the signs of the features $\bX_1, \dots, \bX_d$ so that the output is positively correlated with the features. Assumption \ref{it:ass-r} assumes that this pre-processing has been done, and---for technical reasons---that the correlations are non-zero.
	
	\bigskip
	
	\item\label{it:ass-M} For all $i \neq j$, $M_{ij} =\langle \bX_i, \bX_j \rangle \leq 0$, i.e.,~the features are anti-correlated. 
	
	\bigskip\noindent 
	We assume that once the features have been positively correlated with the output, they are anti-correlated. This assumption is a strong restriction to the class of studied problems and weakening it is left as an open problem. A major motivation for this assumption is that it implies that the trajectories of the DLN dynamics are nondecreasing.
	
	\begin{lem}
		\label{lem:nondecreasing}
		Assume \ref{it:ass-r}-\ref{it:ass-M}. There exists $\varepsilon_0 > 0$ such that for all $\varepsilon \in (0,\varepsilon_0]$, for all $i \in \{1, \dots, d\}$, $\theta_i^{(\varepsilon)}(t)$ is nondecreasing in $t$. 
	\end{lem}
	The proof of this result is postponed to Appendix \ref{sec:proof-prop-nondecreasing}.
\end{enumerate}

As a side comment, note that Assumptions \ref{it:ass-r} and \ref{it:ass-M} jointly constrain the problem to be in the underparametrized regime $n \geq d$.

\begin{proposition}
	\label{prop:positive-definite}
	Assume \ref{it:ass-r}-\ref{it:ass-M}. Then $\bM = \bX^\top \bX$ is positive definite. In particular, as $\bM \in \R^{d \times d}$ and $\bX \in \R^{n\times d}$, we have $n \geq d$.
\end{proposition}
The proof of this result is postponed to Appendix \ref{sec:proof-prop-positive-definite}.

\subsection{Statement of the Results}\label{sec:statement} Our main result (Theorem \ref{thm:main} below) states that the DLN spends long periods of time in the vicinity of fixed points of \eqref{eq:GF-reparametrized}, and describes the times at which transitions occur. To start with, we describe this family of fixed points, using the notations introduced in Section \ref{sec:intro}.


\begin{proposition}
	\label{prop:fixed-points}
	Assume \ref{it:ass-r}-\ref{it:ass-M}. For all $I \subset \{1,\dots,d\}$, there exists a unique $\btheta \geq 0$ fixed point of \eqref{eq:GF-reparametrized} with support $\{ i \in \{1, \dots, d\} \, | \, \theta_i > 0\}$ equal to $I$. We denote this fixed point as $\btheta_*^{(I)}$. Its non-zero coordinates are $(\btheta_*^{(I)})_I = (\bM_{II})^{-1} \br_I$. There are thus $2^d$ fixed points of \eqref{eq:GF-reparametrized}. 
\end{proposition}
The proof that $(\bM_{II})^{-1}$ exists and the proof of the proposition are postponed to Appendix~\ref{sec:proof-prop-fixed}. We give here a high-level intuition. For each $i \in \{1, \dots, d\}$, there are two ways of canceling out the right hand side of \eqref{eq:GF-reparametrized}: either $\theta_i = 0$ or $r_i - \sum_j M_{ij} \theta_j = 0$. For the fixed point $\btheta_*^{(I)}$, the set $I \subset \{1, \dots, d\}$ is the set of coordinates $i$ such that $\theta_{*,i}^{(I)} \neq 0$ and $r_i - \sum_j M_{ij} \theta_{*,j}^{(I)} = 0$; conversely for $i \notin I$, $\theta_{*,i}^{(I)} = 0$.
We say that the coordinates in $I$ are the \emph{active} coordinates of $\btheta_*^{(I)}$.

If no coordinate is active, we obtain the fixed point $\btheta_*^{(\emptyset)} = \bzero$ of \eqref{eq:GF-reparametrized}. If all coordinates are active, $\btheta_*^{(\{1,\dots,d\})} = \bM^{-1}\br$ is the minimum of $f$, thus a fixed point of both gradient flows \eqref{eq:GF-classical} and \eqref{eq:GF-reparametrized}. In Figure \ref{fig:vector-field}, we provide an illustration in dimension $d=2$. We show the vector field defined by the DLN dynamics \eqref{eq:GF-reparametrized}, its $2^d = 4$ fixed points enumerated above and the trajectories $\theta^{\varepsilon}(t)$ for different values of $\varepsilon$. 

\begin{figure}
	\begin{center}
		\includegraphics[width=0.6\linewidth]{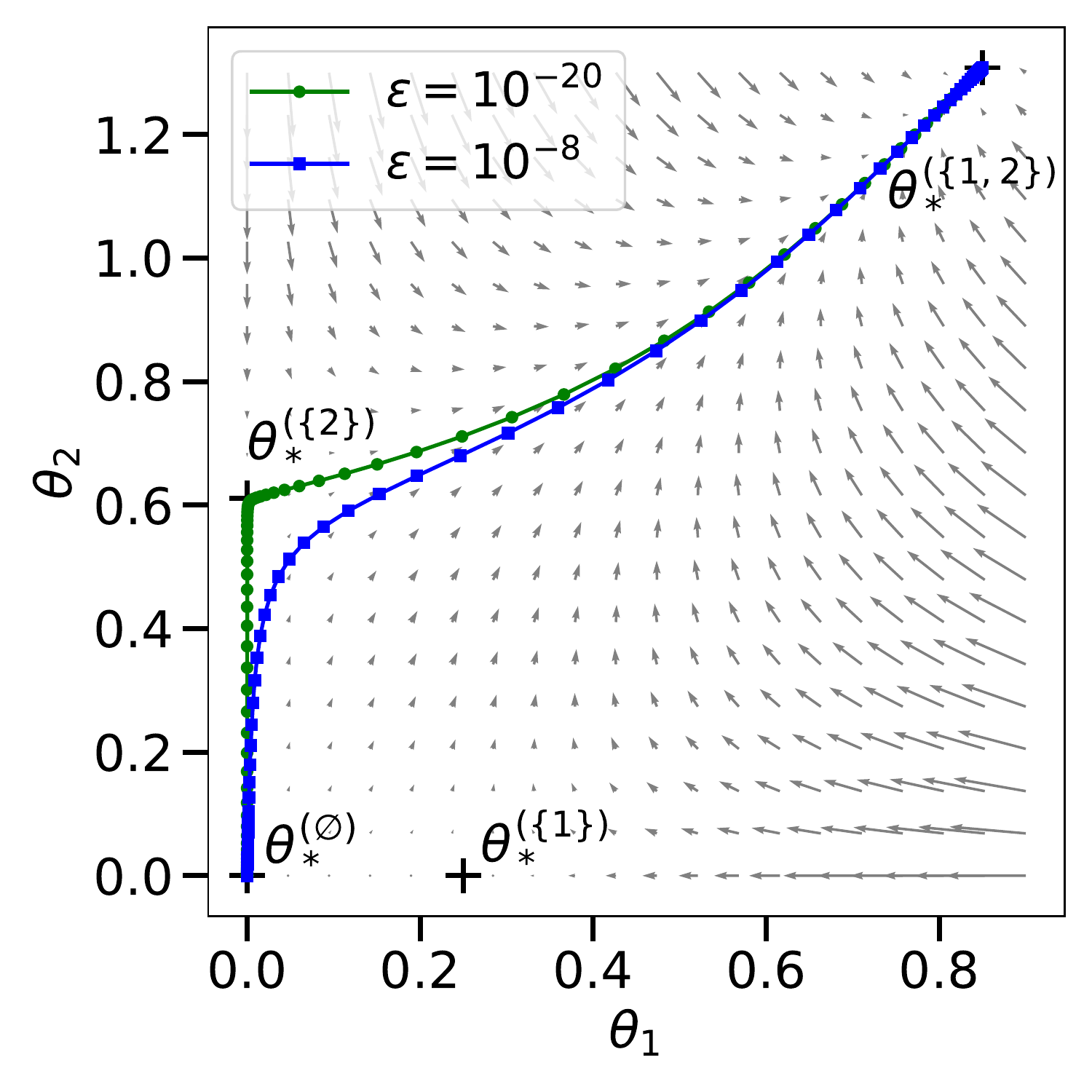}
	\end{center}
	\caption{In dimension $d=2$, we show the vector field $V(\theta_1, \theta_2) = (\theta_1(r_1 - A_{11}\theta_1 - A_{12}\theta_2), \theta_2(r_2 - A_{21}\theta_1 - A_{22}\theta_2))$ associated to the DLN dynamics \eqref{eq:GF-reparametrized} (gray arrows), its fixed points $\theta_*^{(I)}$ for $I \subset \{1,2\}$ (black crosses) and the trajectories of $\theta^{(\varepsilon)}(t)$ for $\varepsilon = 10^{-8}$ (blue) and $\varepsilon = 10^{-20}$ (green). In this simulation, $n = 3$ and the data $\bX \in \R^{n \times d}$, $\by \in \R^n$ is generated randomly with i.i.d.~standard Gaussian entries, conditionally on the event that Assumptions \ref{it:ass-r} and \ref{it:ass-M} hold. The initialization is $\btheta^{(\varepsilon)}(0) = (\varepsilon, \varepsilon)$.}
	\label{fig:vector-field}
\end{figure}

We are now in position to state our main theorem. Recall that for $\varepsilon > 0$, $\btheta^{(\varepsilon)} = \btheta^{(\varepsilon)}(t)$ is the solution of the DLN dynamics \eqref{eq:GF-reparametrized} initialized from $\btheta^{(\varepsilon)}(0) = (C_1 \varepsilon^{k_1}, \dots, C_d \varepsilon^{k_d})$, where $\bC = (C_1, \dots, C_d) > \bzero$ and  $\bk = (k_1, \dots, k_d) > \bzero$ are constants. 

\begin{thm}
	\label{thm:main}
	Assume \ref{it:ass-r}-\ref{it:ass-M}. For $s>0$, define $\bmu(s)$ as the unique minimizer of the regularized and constrained minimization problem 
	\begin{equation}
		\label{eq:main-optim-pb}
		\underset{\btheta \in \R^d, \, \btheta \geq \bzero}{\rm{min}.} \left\{ f(\btheta) + \frac{1}{s} \langle \bk, \btheta \rangle \right\} 
	\end{equation}
	and
	\begin{equation*}
		I(s) = \left\{ i \in \{1, \dots, d\} \, | \,  \mu_i(s) > 0 \right\} \, .
	\end{equation*}
	Then we have the following:
	\begin{enumerate}
		\item\label{it:main-1} The minimizer $\bmu(s)$ is nondecreasing in $s$, i.e., for all $i \in \{1,\dots,d\}$, $\mu_i(s)$ is nondecreasing in $s$. Consequently, $I(s)$ is nondecreasing in $s$ for the inclusion. 
		\item\label{it:main-2} Denote $s_1, \dots, s_q$ the points of discontinuity of the function $s \mapsto I(s)$. For all $s > 0$, $s \neq s_1, \dots, s_q$,
		\begin{equation*}
			\btheta^{(\varepsilon)}\left(s \log \frac{1}{\varepsilon}\right) \xrightarrow[\varepsilon \to 0]{} \btheta_*^{(I(s))} \, .
		\end{equation*}
		Moreover, the convergence is uniform for $s$ in compact subsets of $\R_{>0} \backslash \{s_1, \dots, s_q\}$.
		\item\label{it:main-3} For all $s>0$, 
		\begin{equation*}
			\frac{1}{s \log \frac{1}{\varepsilon}} \int_{0}^{s \log \frac{1}{\varepsilon}} \diff t \,  \btheta^{(\varepsilon)}(t) \xrightarrow[\varepsilon \to 0]{} \bmu(s) \, .
		\end{equation*}
		Moreover, the convergence is uniform for $s$ in compact subsets of $\R_{>0}$.
	\end{enumerate}
\end{thm}
This theorem is proved in Section \ref{sec:proof-thm-main}. It states that $\btheta^{(\varepsilon)}\left(s \log \frac{1}{\varepsilon}\right)$ converges to a piecewise constant function, taking values at the fixed points of the DLN dynamics~\eqref{eq:GF-reparametrized}. Moreover, the set $I(s)$ of active coordinates of the limit is nondecreasing, showing that there are successive coordinate activations. 

\begin{figure}
	\begin{subfigure}{0.48\linewidth}
		\includegraphics[width=\linewidth]{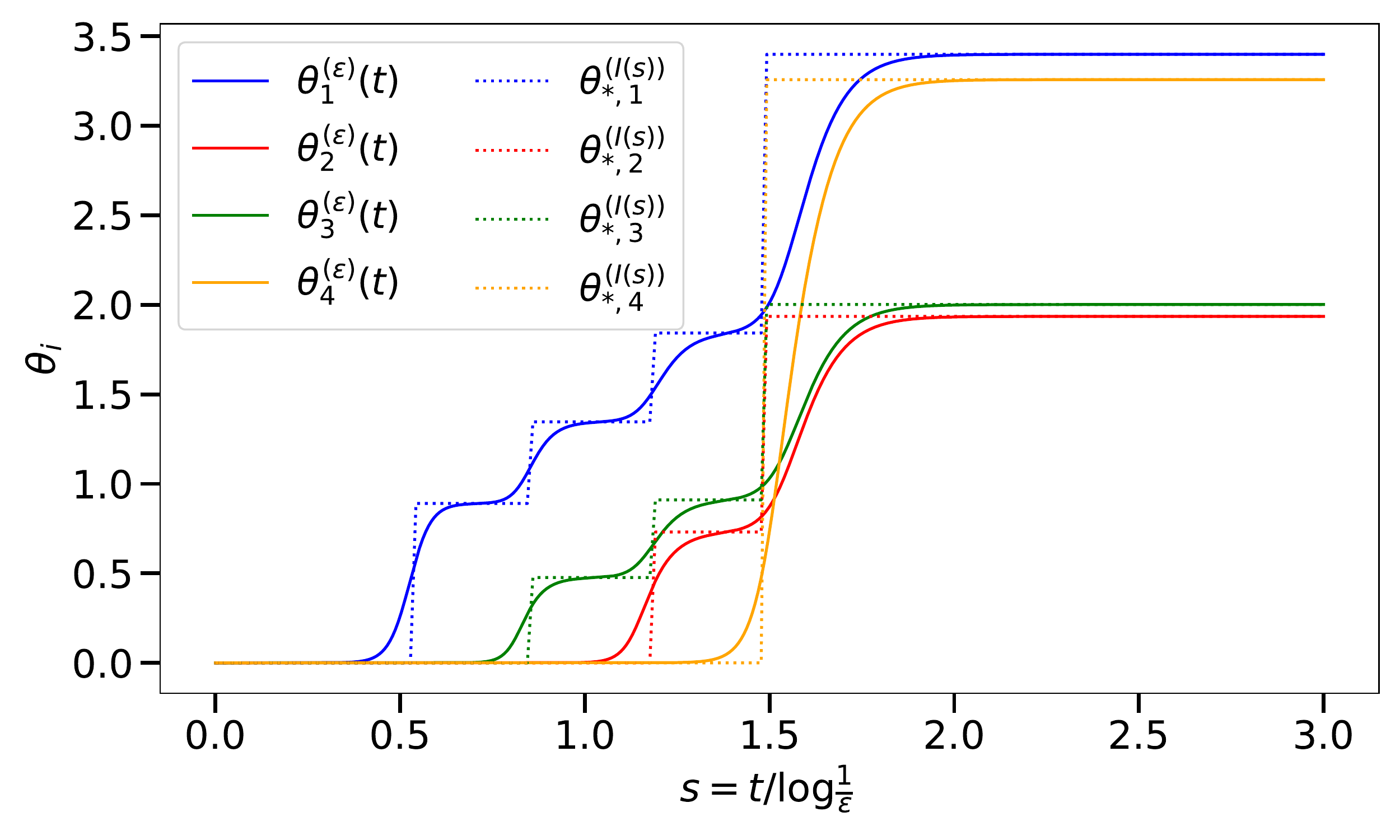}
		\includegraphics[width=\linewidth]{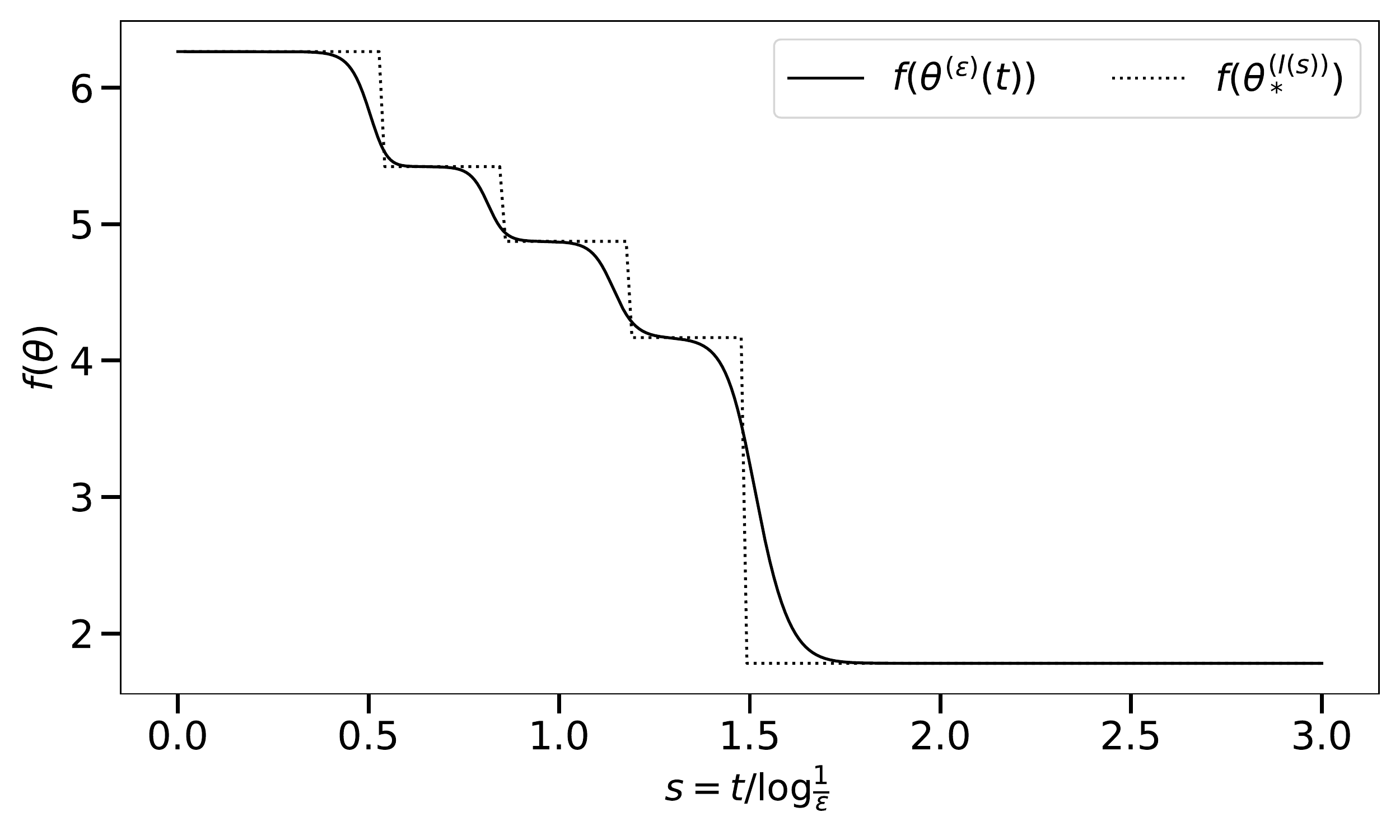}
		\caption{$\varepsilon = 10^{-8}$}
	\end{subfigure}
	\begin{subfigure}{0.48\linewidth}
		\includegraphics[width=\linewidth]{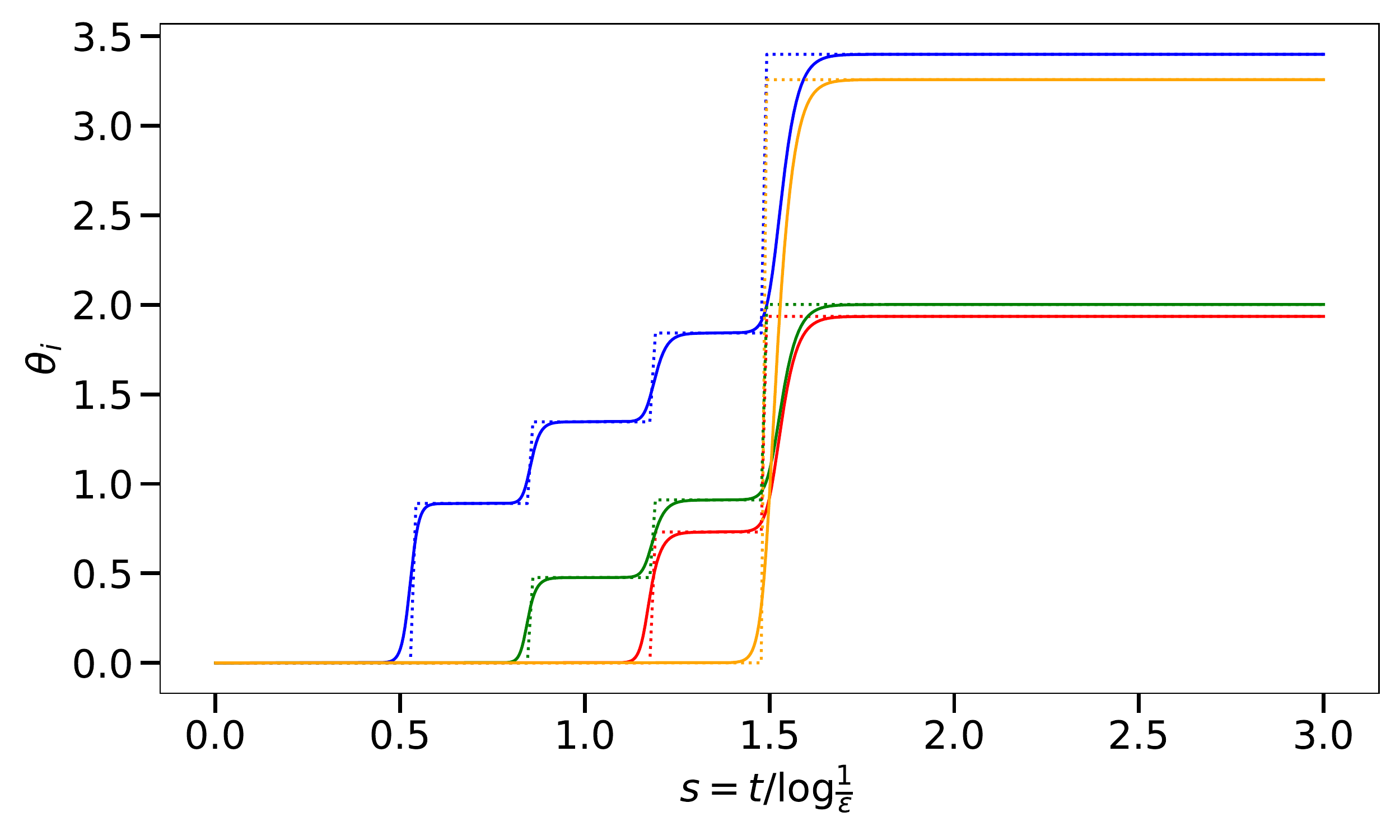}
		\includegraphics[width=\linewidth]{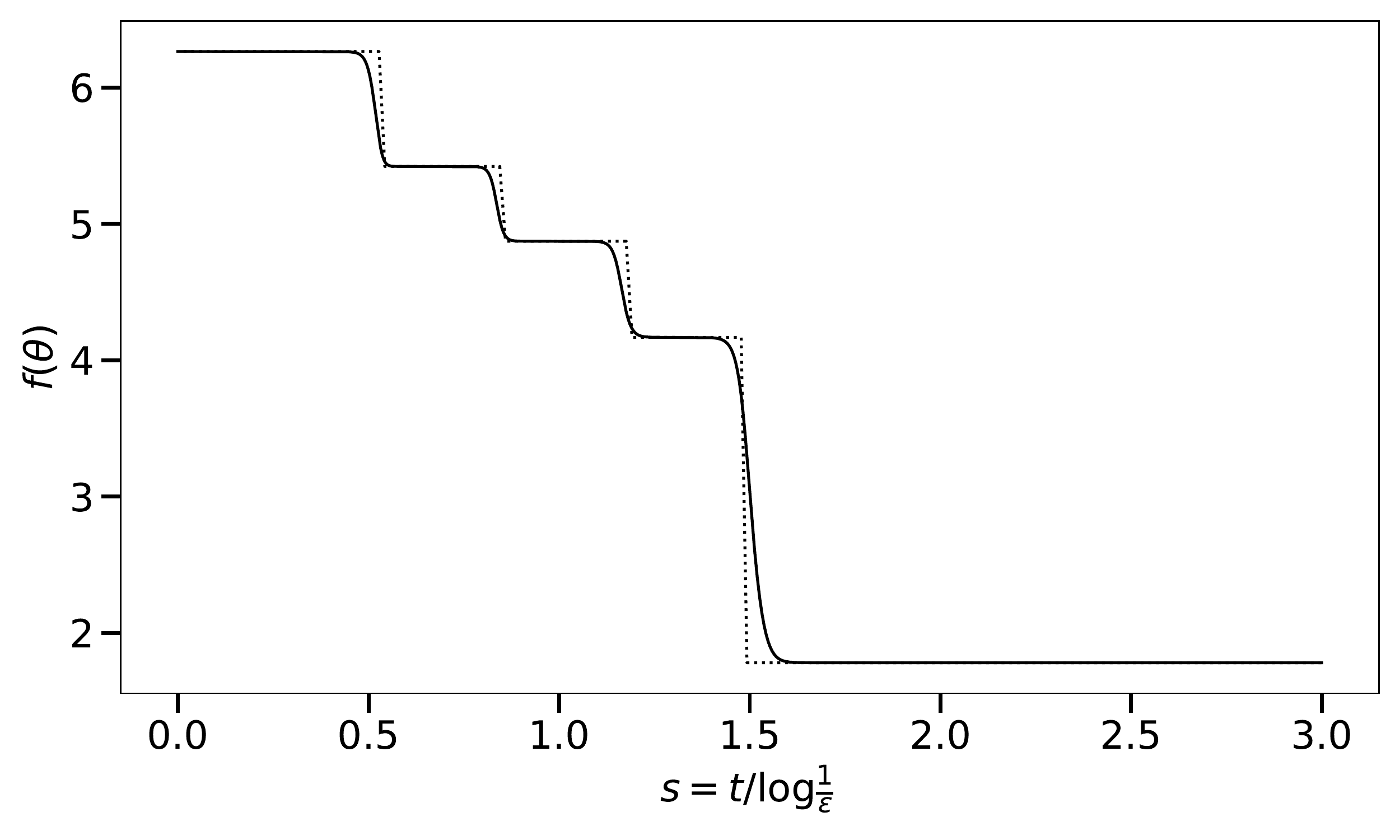}
		\caption{$\varepsilon = 10^{-20}$}
	\end{subfigure}
	\caption{Comparison between the coordinates $\theta^{(\varepsilon)}_i\left(s \log \frac{1}{\varepsilon}\right)$ and their asymptotic approximation $\theta_{*,i}^{(I(s))}$ (upper plots) and between the losses $f\left(\btheta^{(\varepsilon)}\left(s \log \frac{1}{\varepsilon}\right)\right)$ and $f(\btheta_*^{(I(s))})$ (lower plots). The simulations are run with $\varepsilon = 10^{-8}$ (left plots) and $\varepsilon = 10^{-20}$ (right plots). In this simulation, $n = 5$, $d=4$ and the data $\bX \in \R^{n \times d}$, $\by \in \R^n$ is generated randomly with i.i.d.~standard Gaussian entries, conditionally on the event that Assumptions \ref{it:ass-r} and \ref{it:ass-M} hold. The initialization is $\btheta^{(\varepsilon)}(0) = (\varepsilon, \dots, \varepsilon)$ and thus $\bk = (1, \dots, 1)$.}
	\label{fig:main}
\end{figure}

We provide an illustration of these successive coordinate activations in Figure \ref{fig:main}. Note that when a new coordinate is activated, all other coordinates are perturbed. Moreover, as $\varepsilon$ decreases from $10^{-8}$ to $10^{-20}$, one can observe that $\btheta_*^{(I(s))}$ is becoming a sharper approximation of $\btheta^{(\varepsilon)}\left(s \log \frac{1}{\varepsilon}\right)$ and $f(\btheta_*^{(I(s))})$ is becoming a sharper approximation of $f\left(\btheta^{(\varepsilon)}\left(s \log \frac{1}{\varepsilon}\right)\right)$.

The set $I(s)$ of active coordinates of the limit is obtained by solving a regularized and constrained version \eqref{eq:main-optim-pb} of the original optimization problem \eqref{eq:optim-pb}. The non-active constraints at the optimum $\bmu(s)$ correspond to the active coordinates of $\btheta_*^{(I(s))}$. 

The regularization term $+\frac{1}{s} \langle \bk, \btheta \rangle$ in \eqref{eq:main-optim-pb} has a decreasing sparse regularizing effect. The author did not find a finer high-level motivation to explain why $I(s)$ should be defined through \eqref{eq:main-optim-pb}; his insights come only from the proof of the theorem. However, Theorem \ref{thm:main}.\eqref{it:main-3} states a second relation between the DLN dynamics \eqref{eq:GF-reparametrized} and the optimization problem~\eqref{eq:main-optim-pb}: the average $\frac{1}{s \log \frac{1}{\varepsilon}} \int_{0}^{s \log \frac{1}{\varepsilon}} \diff t \,  \btheta^{(\varepsilon)}(t)$ of the trajectory converges to the minimizer $\bmu(s)$ as $\varepsilon\to 0$; said differently, the average of the trajectory computes the regularization path \cite[Section 3]{hastie2009elements} of the regularized optimization problem \eqref{eq:main-optim-pb}.

\bigskip
\begin{remark}
	In Theorem \ref{thm:main}.\eqref{it:main-2}, it is not possible to have uniform convergence in neighborhoods of $s_1, \dots, s_q$ as the functions $\btheta^{(\varepsilon)}\left(s \log \frac{1}{\varepsilon}\right)$ are continuous while $\btheta_*^{(I(s))}$ is discontinuous at $s_1, \dots, s_q$; a uniform convergence would contradict the Arzel\`a theorem \cite[Section IV.6]{dunford1988linear}. 
\end{remark}

The definition of the asymptotic process $\btheta_*^{(I(s))}$ is rather complex as one has to solve an optimization problem. Nevertheless, in the following corollary, we show that it is still possible to deduce a simple expression for the convergence time of $\btheta^{(\varepsilon)}(t)$ to the minimizer $\btheta_*^{(\{1,\dots,d\})}$ of $f$.

\begin{coro}[convergence time to the minimizer] 
	\label{coro:main}
	Assume \ref{it:ass-r}-\ref{it:ass-M}. For all $\eta>0$, denote 
	\begin{equation*}
		\tau_\eta^{(\varepsilon)} = \inf \left\{ t \geq 0 \, \middle\vert \, \left\Vert \btheta^{(\varepsilon)}(t) - \btheta_*^{(\{1,\dots,d\})} \right\Vert \leq \eta \right\}
	\end{equation*}
	the hitting time of the ball centered around the minimizer $\btheta_*^{(\{1,\dots,d\})}$ of $f$ and of radius~$\eta$. Then for $\eta$ small enough, 
	\begin{equation*}
		\frac{\tau_\eta^{(\varepsilon)}}{\log \frac{1}{\varepsilon}} \xrightarrow[\varepsilon \to 0]{} \underset{i\in\{1,\dots,d\}}{\max} \frac{(\bM^{-1}\bk)_i}{(\bM^{-1}\br)_i}  \, .
	\end{equation*}
\end{coro}
This corollary is proved in Section \ref{sec:proof-coro-main}. Note that we describe the hitting time $\tau_\eta^{(\varepsilon)}$ only in the asymptotic limit $\varepsilon \to 0$, and in this limit, $\tau_\eta^{(\varepsilon)}/\log \frac{1}{\varepsilon}$ is independent of $\eta$ (for $\eta$ small enough). This surprising property is due to the fact that the limit trajectory reaches the global minimizer $\btheta_*^{(\{1,\dots,d\})}$ through a last jump of the iterates (see Figure \ref{fig:main}) and the duration of this jump is negligible before $\log \frac{1}{\varepsilon}$.

\section{Related Work}
\label{sec:related}

\noindent
\emph{Diagonal linear networks (DLNs).} Previous studies of DLNs show that their dynamics select sparse estimators in an overparametrized setting \citep{vaskevicius2019implicit,zhao2019implicit,woodworth2020kernel,haochen2021shape,li2021implicit,azulay2021implicit,pesme2021implicit,vivien2022label,nacson2022implicit,chou2021more}. Our work differs from previous studies in two ways: first, we describe the limit of the full trajectory of the dynamics, but second, we are technically restricted to the underparametrized setting. 

Many studies consider the more general quadratic reparametrization $\theta_i = (u_i^2 - v_i^2)/4$ or equivalently $\theta_i = u_i v_i$, which do not constrain $\theta_i$ to be non-negative, while our reparametrization $\theta_i = u_i^2/4$ does. However, under our Assumptions \ref{it:ass-r} and \ref{it:ass-M}, the restriction to non-negative regressors is benign. Heuristically, the regressors $\theta_i$ are nondecreasing (Lemma~\ref{lem:nondecreasing}), thus they do not ``try'' to become negative. We thus claim that under the more general parametrization $\theta_i = (u_i^2 - v_i^2)/4$, the variables $v_i$ would remain negligible and the results would be the same. 

When $\bM = \bX^\top \bX =  \bI_d$, the DLNs dynamics \eqref{eq:GF-reparametrized} are separable across coordinates and can be solved using the logistic equation. In this special case, the activation of a coordinate does not affect the other coordinates. Further, one can check that the coordinates are activated in the decreasing order of the loss decrease that they induce. In \citep{vaskevicius2019implicit,zhao2019implicit,li2021implicit}, a restricted isometry property or incoherence property controls the deviation from this special case. On the contrary, in this paper, we do not make such an assumption and observe richer phenomena. In Figure~\ref{fig:main}, we observe each coordinate activation has a large influence on other active coordinates; moreover, the coordinate introduced last is the second largest coordinate of the optimum $\btheta_*^{(\{1,\dots,4\})}$ and induces the largest loss decrease.

To the best of our knowledge, previous analyses of DLNs have focused on the case where the initializations $\theta_i^{(\varepsilon)}(0) = C_i\varepsilon^{k_i}$ of all coordinates have the same order of magnitude, i.e., $\bk = (1, \dots, 1)$. In this paper, we generalize to $\bk \neq (1, \dots, 1)$: this has the effect of weighting the sparse regularizing term of \eqref{eq:main-optim-pb}.

We believe that the techniques of this paper can be adapted to deeper DLNs, i.e., when $\theta_i \propto u_i^l$, $l > 2$. One would only need to assume additionally that $k_1 = \dots = k_d$. In this case, the time rescaling to have a limiting trajectory would change from $\log 1/\varepsilon$ for $l=2$ to $\varepsilon^{2/l-1}$ for $l > 2$. Moreover, in this latter case, one observes that the effective regularization in Theorem \ref{thm:main} depends on the constants $C_1, \dots, C_d$ (but is still linear in $\btheta$). This is observed by repeating the proof of Section \ref{sec:proof-thm-main}, redefining $w_i^{(\varepsilon)}$ as $(\theta_i/\varepsilon)^{2/l-1}$. We have omitted this adaptation for simplicity.

Finally, we note that when a $\ell^2$ penalization on $\bu$ (or on $\bu$ and $\bv$) is added to $f$, DLNs are related to iterative reweighted least-squares, a reparametrization of the Lasso problem appreciated for computational purposes, see \citep[Section 5]{bach2012optimization} or \citep{poon2021smooth}. However, in this paper, there is no explicit $\ell^2$ penalty on $\bu$ and thus no explicit $\ell^1$ penality on $\btheta$.

\medskip\noindent
\emph{Incremental learning.} Incremental learning describes some learning curves observed in human and machine learning that are almost piecewise constant: they consist of stages where little progress is made, separated by sharp transitions. For instance, this phenomenon occurs in non-diagonal linear networks \citep{saxe2019mathematical,gidel2019implicit,gissin2019implicit,arora2019implicit,chou2020gradient,li2020towards}, in tensor decomposition \citep{ge2021understanding, razin2021implicit,razin2022implicit,hariz2022implicit} and in shallow ReLU networks \citep{boursier2022gradient}. In general, obtaining a mathematical description of the process---of the times of the transitions and the progress made---is mathematically challenging. To the best of our knowledge, existing works obtain a rigorous and complete mathematical description only in ``separable'' cases where the learning dynamics can be separated into several one-dimensional learning dynamics. For instance, \cite{gissin2019implicit} study DLNs but only in the special case $\bM = \bI_d$. As a consequence, a major contribution of our work is to describe precisely some \emph{non-separable} incremental learning dynamics.

\medskip\noindent
\emph{Heteroclinic networks.} From a dynamical systems perspective, the dynamics \eqref{eq:GF-reparametrized} form a heteroclinic network \citep{bakhtin2011noisy}: it has several fixed points (also called \emph{saddle points}) $(\btheta_*^{(I)})_{I \subset \{1,\dots,d\}}$ connected by geodesics of the flow. Such a dynamical system spends large amounts of time in the vicinity of fixed points, with sharp transitions between them. In our case, this is closely related to incremental learning. For our dynamical system, we describe the sequence of visited fixed points and the transition times. The paper of \citet{jacot2021deep} attempted a similar study for linear networks; we prove rigorously such results in the special case of \emph{diagonal} linear networks.

\medskip\noindent
\emph{Lotka--Volterra equations.} To finish, we note that the quadratic system \eqref{eq:GF-reparametrized} of ordinary differential equations are Lotka--Volterra (LV) equations \citep{hofbauer_sigmund_1998,baigent2017lotka}. Traditionally, in mathematical biology, these equations represent the evolution the populations sizes $\theta_1, \dots, \theta_d$ of $d$~interacting species. The parameter $r_i$ represents the intrinsic growth of population $i$ while the parameter $M_{ij}$ represents the interaction between populations $i$ and $j$. 

This point of view, and in particular the paper of \citet{goh1979stability}, inspired the author to use the function \eqref{eq:lyapunov-LV} in the proof of Theorem \ref{thm:main}. In general, our paper can be interpreted as a study of LV equations for cooperative and symmetric interactions from infinitesimal initial population sizes. To the best of our knowledge, such a study did not exist in the literature on LV equations; its implications will be the subject of a forthcoming paper. 

\section{Proof of Theorem \ref{thm:main}}
\label{sec:proof-thm-main}

In this proof, we use both time variables $t$ and $s$, with $t = s \log \frac{1}{\varepsilon}$. As it is frequent in the literature on ordinary differential equations (ODEs), we abusively use the same notation for functions of $t$ and $s$. For instance, by convention, $\btheta^{(\varepsilon)}(s) := \btheta^{(\varepsilon)}(t)$ with $t = s \log \frac{1}{\varepsilon}$. In fact, we often drop the dependence on time. For instance, $\btheta^{(\varepsilon)} := \btheta^{(\varepsilon)}(s) = \btheta^{(\varepsilon)}(t)$. 

We start with a crude estimate of the trajectories $\btheta^{(\varepsilon)}(t)$ that is useful several times later in the proof. 

\begin{lem}
	\label{lem:bounded}
	The trajectory $\btheta^{(\varepsilon)}(t)$ is bounded uniformly for $\varepsilon \in (0,1]$ and $t \in \R_{\geq 0}$, i.e.,~there exists a constant $B > 0$ such that $\forall \varepsilon\in (0,1], \forall t \in \R_{\geq 0}, \Vert \btheta^{(\varepsilon)}(t)\Vert \leq B$.
\end{lem}

\begin{proof}
	As Equation \eqref{eq:GF-reparametrized} is a (reparametrized) gradient flow of $f$, $f$ is a Lyapunov function, i.e.,
	\begin{equation*}
		\frac{\diff }{\diff t} f(\btheta) = \sum_{i=1}^{d} \frac{\diff f}{\diff \theta_i} \frac{\diff \theta_i}{\diff t} \underset{\eqref{eq:GF-reparametrized}}{=} - \sum_{i=1}^{d}\theta_i \left(\frac{\diff f}{\diff \theta_i}\right)^2 \leq 0 \, .
	\end{equation*}
	Thus for all $\varepsilon \in (0,1]$, for all $t \geq 0$, 
	\begin{equation}
		\label{eq:aux-2}
		f(\btheta^{(\varepsilon)}(t)) \leq f(\btheta^{(\varepsilon)}(0)) \leq \sup_{\varepsilon \in (0,1]} f(\btheta^{(\varepsilon)}(0)) \, .
	\end{equation}
	This supremum is finite as $f$ is continuous and $\btheta^{(\varepsilon)}(0)$ is uniformly bounded for $\varepsilon \in (0,1]$. Further, as $\bM$ is positive definite (Proposition \ref{prop:positive-definite}), 
	\begin{equation*}
		f(\btheta) = \frac{1}{2} \Vert \by \Vert^2 - \langle \br, \btheta \rangle + \frac{1}{2} \langle \btheta,  \bM \btheta \rangle \to \infty \qquad \text{as }\Vert \btheta \Vert \to \infty \, .
	\end{equation*}
	Thus the uniform bound \eqref{eq:aux-2} implies a uniform bound on $\Vert \btheta^{(\varepsilon)}(t) \Vert$.
\end{proof}

The central idea of the proof of Theorem \ref{thm:main} is to keep track of the size of the coordinates of $\btheta^{(\varepsilon)}(t)$, in order to be able to determine which coordinates of $\btheta^{(\varepsilon)}(t)$ are activated depending on time~$t$. More precisely, define 
\begin{equation}
	\label{eq:def-w}
	w_i^{(\varepsilon)} = \frac{\log \theta_i^{(\varepsilon)}}{\log \varepsilon} \, .
\end{equation}
Equivalently, this gives $\theta_i^{(\varepsilon)} = \varepsilon^{w_i^{(\varepsilon)}}$. This logarithmic transformation of the coordinates is particularly convenient because its time derivative is affine in $\btheta^{(\varepsilon)}$:
\begin{align*}
	\frac{\diff w_i^{(\varepsilon)}}{\diff s} = \frac{\diff t}{\diff s} \frac{\diff w_i^{(\varepsilon)}}{\diff t} = \left(\log \frac{1}{\varepsilon} \right) \frac{1}{\theta_i^{(\varepsilon)} \log \varepsilon} \frac{\diff \theta_i^{(\varepsilon)}}{\diff t} \underset{\eqref{eq:GF-reparametrized}}{=} \sum_{j=1}^d M_{ij} \theta_j^{(\varepsilon)} - r_i \, ,
\end{align*}
or, using the vector notation $\bw^{(\varepsilon)} = (w_1^{(\varepsilon)}, \dots, w_d^{(\varepsilon)})$,
\begin{align*}
	\frac{\diff \bw^{(\varepsilon)}}{\diff s}  = \bM \btheta^{(\varepsilon)} - \br \, .
\end{align*}
Our proof technique determines the limit of $\bw^{(\varepsilon)}$ as $\varepsilon \to 0$. The limit is described as the Lagrange multiplier of an optimization problem closely related to \eqref{eq:main-optim-pb}. We start with a brief reminder on duality in optimization. 

\begin{proposition}
	\label{prop:optim-LCP}
	Let $\bq, \bz \in \R^d$. The two following statements are equivalent:
	\begin{enumerate}
		\item $\bz$ is the unique minimizer of the constrained optimization problem 
		\begin{equation}
			\label{eq:optim}
			\underset{\btheta \in \R^d, \, \btheta \geq \bzero}{\rm{min}.} \left\{ \langle \bq, \btheta \rangle + \frac{1}{2} \langle \btheta, \bM \btheta \rangle \right\} \, .
		\end{equation}
		\item There exists $\bw \in \R^d$ such that $(\bw,\bz)$ is the unique solution of 
		\begin{align*}
			&\bw = \bq + \bM \bz \, , \\
			&\bw \geq \bzero \, , \qquad \bz \geq \bzero \, , \qquad \bw^\top \bz = 0 \, .
		\end{align*}
		The four conditions above form a so-called \emph{linear complementarity problem (LCP)}, where $\bq$ and $\bM$ are the parameters and $\bw$ and $\bz$ are the variables. 
	\end{enumerate}
\end{proposition}

The linear complementarity problem is widely studied; see for instance the monograph of~\citet{cottle2009linear} or Appendix \ref{sec:properties}. In this connection with the quadratic programming problem \eqref{eq:optim}, the variable~$\bw$ should be seen as the Lagrange multiplier associated to the constraint $\btheta \geq \bzero$. The LCP expresses the Karush–Kuhn–Tucker (KKT) conditions for optimality to hold. More precisely, $\bw = \bq + \bM \bz$ is a condition of stationarity of the Lagrangian; $\bw \geq \bzero$ and $\bz \geq \bzero$ are respectively dual and primal feasibility conditions; and $\bw^\top \bz = 0$ is a complementarity slackness condition. Put together, the conditions $\bw \geq \bzero$, $\bz \geq \bzero$ and $\bw^\top \bz = 0$ impose that for all $i \in \{1,\dots,d\}$, either $w_i =0$ or $z_i = 0$. 

Proposition \ref{prop:optim-LCP} is classical; nevertheless we detail the appropriate references in Appendix~\ref{sec:proof-prop-optim-LCP}. 

We are now in position to describe the asymptotic behavior of $w_i^{(\varepsilon)} = \frac{\log \theta_i^{(\varepsilon)}}{\log \varepsilon}$. Define 
\begin{equation*}
	\bz^{(\varepsilon)}(s) = \int_{0}^{s} \diff u \,  \btheta^{(\varepsilon)}(u) \, .
\end{equation*}

\begin{proposition}
	\label{prop:uniform-convergences}
	Let $(\bw(s),\bz(s))$ be the unique solution of the linear complementarity problem 
	\begin{align}
		\label{eq:LCP-uniform}
		\begin{split}
			&\bw = \bk-s\br + \bM \bz \, , \\ 
			&\bw \geq \bzero \, , \qquad \bz \geq \bzero \, , \qquad \bw^\top \bz = 0 \, .
		\end{split}
	\end{align}
	Then $\bw^{(\varepsilon)}(s)$ and $\bz^{(\varepsilon)}(s)$ converge respectively to $\bw(s)$ and $\bz(s)$ as $\varepsilon \to 0$, uniformly on compact subsets of $\R_{\geq 0}$.
\end{proposition}

\begin{proof}
	In this proof, we define $\varepsilon_0 > 0$ as in Lemma \ref{lem:nondecreasing} and $B > 0$ as in Lemma \ref{lem:bounded}. Further we take $\varepsilon_1 = \min\left(\varepsilon_0, \frac{1}{2}\right)$ and assume that $\varepsilon \leq \varepsilon_1$.
	
	Fix $S > 0$ and define the continuous functions 
	\begin{align*}
		&\bvarphi^{(\varepsilon)} : s \in [0,S] \mapsto \left(\bw^{(\varepsilon)}(s), \bz^{(\varepsilon)}(s)\right) \in (\R^d)^2 \, , \\
		&\bvarphi : s \in [0,S] \mapsto \left(\bw(s), \bz(s)\right) \in (\R^d)^2 \, .
	\end{align*}
	We want to show that $\bvarphi^{(\varepsilon)} \to \bvarphi$ uniformly as $\varepsilon \to 0$. First, we use the Arzel\`a--Ascoli theorem to check that the set $\{\bvarphi^{(\varepsilon)}, \varepsilon \in (0, \varepsilon_1)\}$ is relatively compact in the space of continuous functions from $[0,S]$ to $(\R^d)^2$. The reader can consult the monograph of \citet[Section IV.6]{dunford1988linear} for a reference on the Arzel\`a--Ascoli theorem for real-valued functions; the multidimensional extension is straightforward. 
	\begin{itemize}
		\item For $\varepsilon \in (0, \varepsilon_1)$, $s \in [0,S]$, we bound $\Vert \bvarphi^{(\varepsilon)}(s) \Vert^2 = \Vert \bw^{(\varepsilon)}(s) \Vert^2 + \Vert \bz^{(\varepsilon)}(s) \Vert^2$. We bound the two terms separately.
		\begin{itemize}
			\item From Lemmas \ref{lem:nondecreasing} and \ref{lem:bounded}, 
			\begin{align*}
				C_i \varepsilon^{k_i} = \theta_i^{(\varepsilon)}(0) \leq \theta_i^{(\varepsilon)}(s) \leq B \, ,
			\end{align*}
			thus ($\log \varepsilon < 0$),
			\begin{align*}
				\frac{\log B}{\log \varepsilon} \leq w_i^{(\varepsilon)} = \frac{\log \theta_i^{(\varepsilon)}}{\log \varepsilon} \leq \frac{\log C_i}{\log \varepsilon} + k_i \, .
			\end{align*}
			As $\varepsilon\leq 1/2$, $\log \varepsilon$ is bounded away from $0$. This shows that $w_i^{(\varepsilon)}$ is bounded uniformly for $\varepsilon \in (0,\varepsilon_1]$ and $s \in [0,S]$. 
			\item From Lemma \ref{lem:bounded},
			\begin{align*}
				\Vert \bz^{(\varepsilon)}(s) \Vert \leq \int_{0}^{s} \diff u \, \Vert \btheta^{(\varepsilon)}(u) \Vert \leq sB \leq SB \, .
			\end{align*}
		\end{itemize}
		The two points above show that $\Vert \bvarphi^{(\varepsilon)}(s) \Vert^2$ is bounded uniformly for $\varepsilon \in (0,\varepsilon_1]$ and $s \in [0,S]$.
		\item The square norm of the derivative
		\begin{align*}
			\left\Vert  \frac{\diff \bvarphi^{(\varepsilon)}}{\diff s}(s)\right\Vert^2 &= \left\Vert  \frac{\diff \bw^{(\varepsilon)}}{\diff s}(s)\right\Vert^2 + \left\Vert  \frac{\diff \bz^{(\varepsilon)}}{\diff s}(s)\right\Vert^2 \\
			&= \left\Vert \bM \btheta^{(\varepsilon)}(s)-\br \right\Vert^2 + \left\Vert \btheta^{(\varepsilon)}(s)\right\Vert^2
		\end{align*} 
		is bounded uniformly for $\varepsilon \in (0,\varepsilon_1]$ and $s \in [0,S]$ by Lemma \ref{lem:bounded}. Thus the set $\{\bvarphi^{(\varepsilon)}, \varepsilon \in (0, \varepsilon_1)\}$ is equicontinuous. 
	\end{itemize}
	The two points above show that we can apply the Arzel\`a--Ascoli theorem: $\{\bvarphi^{(\varepsilon)}, \varepsilon \in (0, \varepsilon_1)\}$ is relatively compact in the space of continuous functions from $[0,S]$ to $(\R^d)^2$. To conclude on Proposition \ref{prop:uniform-convergences}, it is then sufficient to show that the only subsequential uniform limit of $\bvarphi^{(\varepsilon)}$ as $\varepsilon\to 0$ is $\bvarphi$. 
	
	Let $\bvarphi' = (\bw',\bz')$ be a subsequential uniform limit of $\bvarphi^{(\varepsilon)} = (\bw^{(\varepsilon)}, \bz^{(\varepsilon)})$ as $\varepsilon\to 0$. There exists $\varepsilon(n) \in (0, \varepsilon_1]$ such that $\varepsilon(n) \to 0$ and $\bvarphi^{(\varepsilon(n))} \to \bvarphi'$ uniformly as $n \to \infty$. Then $\bw^{(\varepsilon(n))} \to \bw'$ and $\bz^{(\varepsilon(n))} \to \bz'$ uniformly as $n \to \infty$. We check that $(\bw',\bz')$ is a solution of the LCP \eqref{eq:LCP-uniform}.
	\begin{itemize}
		\item First,
		\begin{align*}
			\frac{\diff}{\diff s} \left(\bw^{(\varepsilon(n))}(s) - \bk + s\br - \bM \bz^{(\varepsilon(n))}(s)\right) &= \frac{\diff \bw^{(\varepsilon(n))}}{\diff s}(s) + \br    - \bM \frac{\diff \bz^{(\varepsilon(n))}}{\diff s}(s) \\
			&= \bM \btheta^{(\varepsilon(n))}(s) - \br + \br -\bM \btheta^{(\varepsilon(n))}(s) \\
			&= \bzero \, .
		\end{align*}
		Thus $\bw^{(\varepsilon(n))}(s) - \bk + s\br - \bM \bz^{(\varepsilon(n))}(s)$ is constant in $s$, equal to its initial value $\bw^{(\varepsilon(n))}(0) -\bk$. Moreover, for all $i$, 
		\begin{align*}
			w_i^{(\varepsilon(n))}(0) -k_i = \frac{\log C_i\varepsilon(n)^{k_i}}{\log \varepsilon(n)} - k_i = \frac{\log C_i}{\log \varepsilon(n)} \xrightarrow[n \to \infty]{} 0 \, .
		\end{align*}
		Thus $\bw^{(\varepsilon(n))}(0) -\bk = \bw^{(\varepsilon(n))}(s) - \bk + s\br - \bM \bz^{(\varepsilon(n))}(s) \to \bzero$ as $n \to \infty$. By identification of the limit, we have $\bw'(s) - \bk + s\br - \bM \bz'(s) = \bzero$.
		\item Second, using Lemma \ref{lem:bounded} and that $\log \varepsilon(n) < 0$,
		\begin{align*}
			w_i^{(\varepsilon(n))}(s) = \frac{\log \theta_i^{(\varepsilon(n))}(s)}{\log \varepsilon(n)} \geq \frac{\log B}{\log \varepsilon(n)}  \, ,
		\end{align*}
		thus taking $n \to \infty$, we obtain $\bw' \geq \bzero$.
		\item Third, we have $\bz^{(\varepsilon(n))}(s) \geq \bzero$ trivially from the definition, and thus taking $n \to \infty$, we obtain $\bz' \geq \bzero$.
		\item Finally,
		\begin{align*}
			\left\vert \bw^{(\varepsilon(n))}(s)^\top \bz^{(\varepsilon(n))}(s)\right\vert &\leq \sum_{i} \left\vert w^{(\varepsilon(n))}_i(s) z_i^{(\varepsilon(n))}(s) \right\vert  
			\\
			&= \sum_{i} \frac{|\log \theta_i^{(\varepsilon(n))}(s)|}{|\log \varepsilon(n)|} \int_{0}^{s} \diff u \, \theta_i^{(\varepsilon(n))}(u) \\
			&\leq \frac{s}{|\log \varepsilon(n)|} \sum_{i} |\log \theta_i^{(\varepsilon(n))}(s)| \theta_i^{(\varepsilon(n))}(s) \, ,
		\end{align*}
		where in this last inequality we use Lemma \ref{lem:nondecreasing}. The function $x \mapsto |\log x| x$ can be continuously extended in $0$; it is thus bounded on $[0,B]$. Thus
		\begin{align*}
			\left\vert \bw^{(\varepsilon(n))}(s)^\top \bz^{(\varepsilon(n))}(s)\right\vert \leq \frac{sd}{|\log \varepsilon(n)|} \max_{x \in [0,B]} |\log x| x \, .
		\end{align*}
		Taking $n \to \infty$, we obtain $\bw'(s)^\top \bz'(s) = 0$. 
	\end{itemize}
	The four points above show that for all $s \in [0,S]$, $(\bw'(s),\bz'(s))$ is a solution of the LCP~\eqref{eq:LCP-uniform}. As the solution is unique (Proposition \ref{prop:unicity-sol-lcp}), $\bvarphi' = (\bw',\bz') = (\bw,\bz) = \bvarphi$. Thus $\bvarphi$ is the unique subsequent limit of $\bvarphi^{(\varepsilon)}$ as $\varepsilon \to 0$. Thus $\bvarphi^{(\varepsilon)} \xrightarrow[\varepsilon \to 0]{} \bvarphi$ uniformly on $[0,S]$. 
\end{proof}

We now show how the proof of Theorem \ref{thm:main} essentially follows from Proposition \ref{prop:uniform-convergences}.

\begin{proof}[Proof of Theorem \ref{thm:main}]
	First, note that it is shown in Proposition \ref{prop:optim-LCP} that there is a unique minimizer to \eqref{eq:main-optim-pb}. We continue by proving successively the three points of the theorem.
	\begin{enumerate}
		\item[\eqref{it:main-1}] The fact that $\bmu(s)$ is nondecreasing follows from the connection with the LCP (Proposition \ref{prop:optim-LCP}) and the antitonicity property of the solution of the LCP (Proposition~\ref{prop:antitonicity}). We detail this argument.
		
		Recall that $\bmu(s)$ is the unique minimizer of 
		\begin{equation*}
			\underset{\btheta \in \R^d, \, \btheta \geq \bzero}{\rm{min}.} \left\{ f(\btheta) + \frac{1}{s} \langle \bk, \btheta \rangle = \frac{1}{2} \Vert \by \Vert^2 + \left\langle \frac{1}{s}\bk-\br, \btheta \right\rangle + \frac{1}{2} \langle \btheta,  \bM \btheta \rangle \right\} 
		\end{equation*}
		Thus by Proposition \ref{prop:optim-LCP}, there exists $\bv(s) \in \R^d$ such that $(\bv(s),\bmu(s))$ is the unique solution of the LCP 
		\begin{align*}
			\begin{split}
				&\bv = \frac{1}{s}\bk-\br + \bM \bmu \, , \\
				&\bv \geq \bzero \, , \qquad \bmu \geq \bzero \, , \qquad \bv^\top \bmu = 0 \, .
			\end{split}
		\end{align*}
		Let us make a side remark for later purposes. We rewrite the LCP above as 
		\begin{align*}
			\begin{split}
				&s\bv = \bk-s\br + \bM (s\bmu) \, , \\
				&s\bv \geq \bzero \, , \qquad s\bmu \geq \bzero \, , \qquad (s\bv)^\top (s\bmu) = 0 \, .
			\end{split}
		\end{align*}
		Thus $(s\bv(s), s\bmu(s))$ is a solution of the LCP \eqref{eq:LCP-uniform}, of which $(\bw(s),\bz(s))$ is also the solution. By unicity of the solution of the LCP (Proposition \ref{prop:unicity-sol-lcp}),
		\begin{equation}
			\label{eq:identification-LCPs}
			\bz(s) = s\bmu(s) \, .
		\end{equation}
		
		We now return to the proof of point \eqref{it:main-1} of the theorem. Consider $s_1 \leq s_2$. Then $(\bv(s_1), \bmu(s_1))$ (resp.~$(\bv(s_2), \bmu(s_2))$) is the unique solution of the LCP with parameters $\bq^{(1)} = \frac{1}{s_1}\bk-\br$ (resp.~$\bq^{(2)} = \frac{1}{s_2}\bk-\br$) and~$\bM$. Note that as $\bk \geq \bzero$, $\bq^{(1)} \geq \bq^{(2)}$. As $\bM$ is symmetric positive definite (Proposition~\ref{prop:positive-definite}) with non-positive off-diagonal entries (Assumption \ref{it:ass-M}), we apply the antitonicity property (Proposition~\ref{prop:antitonicity}) and obtain that $\bmu(s_1) \leq \bmu(s_2)$. 
		
		\medskip
		\item[\eqref{it:main-3}] Abusing again on the notation of the time indexations, we have
		\begin{align*}
			\frac{1}{s \log \frac{1}{\varepsilon}} \int_{0}^{s \log \frac{1}{\varepsilon}} \diff t \,  \btheta^{(\varepsilon)}(t) = \frac{1}{s } \int_{0}^{s } \diff u \,  \btheta^{(\varepsilon)}(u) = \frac{1}{s} \bz^{(\varepsilon)}(s) \, .
		\end{align*}
		By Proposition \ref{prop:uniform-convergences}, $\bz^{(\varepsilon)}(s)$ converges to $\bz(s)$ uniformly on compact subsets of $\R_{\geq 0}$. Thus $\frac{1}{s}\bz^{(\varepsilon)}(s)$ converges uniformly to $\frac{1}{s} \bz(s)$ on compact subsets of $\R_{> 0}$. (Note that as the factor $1/s$ diverges at $0$, we need to consider compact subsets bounded away from $0$.) As $\frac{1}{s} \bz(s) = \bmu(s)$ (Equation \eqref{eq:identification-LCPs}), this concludes point~\eqref{it:main-3} of the theorem.
		
		\medskip
		
		\item[\eqref{it:main-2}] Fix $p \in \{1, \dots, q, q+1\}$ and $u,v \in \R_{>0}$ such that $s_{p-1} < u < v < s_p$ (with the conventions $s_0 = 0$, $s_{q+1} = \infty$). Let $I$ be the constant value of $I(s)$ for $s \in (s_{p-1}, s_p)$. To prove point~\eqref{it:main-2} of the theorem, it is sufficient to show $\btheta^{(\varepsilon)}\left(s \log \frac{1}{\varepsilon}\right) \xrightarrow[\varepsilon \to 0]{} \btheta_*^{(I)} $ uniformly for $s \in [u,v]$.
		
		For $\btheta \in \R^d$, define 
		\begin{equation}
			\label{eq:lyapunov-LV}
			V(\btheta) = \sum_{i\in I} \left(\theta_i - \theta_{*,i}^{(I)} \log \theta_i\right) \, .
		\end{equation}
		The function $V$ is inspired from a Lyapunov function used in the study of Lotka--Volterra equations \citep{goh1979stability}. Here, we show that $V$ is ``almost'' decreasing on the interval $(s_{p-1}, s_p)$. We first compute
		\begin{align}
			\label{eq:aux-3}
			\frac{\diff}{\diff t} V(\btheta^{(\varepsilon)}) &\underset{\eqref{eq:GF-reparametrized}}{=} \sum_{i \in I} \left(\theta_i^{(\varepsilon)} - \theta_{*,i}^{(I)}\right) \left(r_i - \left(\bM\btheta^{(\varepsilon)}\right)_i\right) \, .
		\end{align}
		Using the definition of $\btheta_*^{(I)}$ in Proposition \ref{prop:fixed-points}, we have 
		\begin{equation*}
			(\bM \btheta_*^{(I)})_I = \bM_{II}\btheta_{*,I}^{(I)} + \bM_{II^c}\btheta_{*,I^c}^{(I)} = \bM_{II}\bM_{II}^{-1} \br_I + \bM_{II^c}\bzero = \br_I \, ,
		\end{equation*}
		thus we can rewrite \eqref{eq:aux-3} as 
		\begin{align*}
			\frac{\diff}{\diff t} V(\btheta^{(\varepsilon)}) &= \sum_{i \in I} \left(\theta_i^{(\varepsilon)} - \theta_{*,i}^{(I)}\right) \left(\left(\bM\btheta_*^{(I)}\right)_i - \left(\bM\btheta^{(\varepsilon)}\right)_i\right) \\
			&= - \left\langle \btheta^{(\varepsilon)} - \btheta_*^{(I)} , \bM (\btheta^{(\varepsilon)} - \btheta_*^{(I)} ) \right\rangle - \sum_{i \notin I} \theta_i^{(\varepsilon)}  \left(\left(\bM\btheta_*^{(I)}\right)_i - \left(\bM\btheta^{(\varepsilon)}\right)_i\right) \, .
		\end{align*}
		Fix $u',v'$ such that $s_{p-1} < u' < u < v < v' < s_p$. We integrate the equality above for $s \in [u',v']$:
		\begin{align*}
			V(\btheta^{(\varepsilon)}(v')) - V(\btheta^{(\varepsilon)}(u')) &= \int_{u'}^{v'} \diff s \frac{\diff}{\diff s} V(\btheta^{(\varepsilon)}) \\
			&= \left(\log \frac{1}{\varepsilon}\right) \int_{u'}^{v'} \diff s \frac{\diff}{\diff t} V(\btheta^{(\varepsilon)}) \\ 
			&= \left(\log \frac{1}{\varepsilon}\right) \Bigr[-\int_{u'}^{v'} \diff s\left\langle \btheta^{(\varepsilon)} - \btheta_*^{(I)} , \bM (\btheta^{(\varepsilon)} - \btheta_*^{(I)} ) \right\rangle \\
			&\qquad\qquad\qquad- \sum_{i \notin I} \int_{u'}^{v'} \diff s \,  \theta_i^{(\varepsilon)}  \left(\left(\bM\btheta_*^{(I)}\right)_i - \left(\bM\btheta^{(\varepsilon)}\right)_i\right)\Bigr] \, .
		\end{align*}
		Denote $\lambda_{\min}(\bM)$ the minimal eigenvalue of $\bM$. By Proposition \ref{prop:positive-definite}, $\lambda_{\min}(\bM) > 0$.
		\begin{align}
			\label{eq:aux-5}
			\begin{split}
				\int_{u'}^{v'} \diff s \left\Vert \btheta^{(\varepsilon)} - \btheta_*^{(I)} \right\Vert^2 &\leq \frac{1}{\lambda_{\min}(\bM)} \int_{u'}^{v'} \diff s\left\langle \btheta^{(\varepsilon)} - \btheta_*^{(I)} , \bM (\btheta^{(\varepsilon)} - \btheta_*^{(I)} ) \right\rangle \\
				&= \frac{1}{\lambda_{\min}(\bM)} \Bigr[ \frac{1}{\log \frac{1}{\varepsilon}} \left(V\left(\btheta^{(\varepsilon)}(u')\right) - V\left(\btheta^{(\varepsilon)}(v')\right)\right) \\
				&\qquad\qquad\qquad- \sum_{i \notin I} \int_{u'}^{v'} \diff s \, \theta_i^{(\varepsilon)}  \left(\left(\bM\btheta_*^{(I)}\right)_i - \left(\bM\btheta^{(\varepsilon)}\right)_i\right)\Bigr] \, .
			\end{split}
		\end{align}
		Take $\varepsilon_0 > 0$ as defined in Lemma \ref{lem:nondecreasing} and now assume $\varepsilon \leq \min(1,\varepsilon_0)$ so that both Lemmas \ref{lem:nondecreasing} and \ref{lem:bounded} apply. Then we have the following estimates:
		\begin{itemize}
			\item Fix $i \notin I$. From Lemma \ref{lem:bounded}, there exists a constant $C$ independent of $\varepsilon$ such that 
			\begin{align*}
				\left\vert \int_{u'}^{v'} \diff s \,  \theta_i^{(\varepsilon)}  \left(\left(\bM\btheta_*^{(I)}\right)_i - \left(\bM\btheta^{(\varepsilon)}\right)_i\right) \right\vert &\leq C \int_{u'}^{v'} \diff s \,  \theta_i^{(\varepsilon)} \\
				&= C \left(v' \frac{1}{v'} \int_{0}^{v'} \diff s \,  \theta_i^{(\varepsilon)} - u' \frac{1}{u'} \int_{0}^{u'} \diff s \,  \theta_i^{(\varepsilon)} \right)
			\end{align*}
			Using Theorem \ref{thm:main}.\eqref{it:main-3}, this last quantity converges as $\varepsilon \to 0$ to $v' \mu_i(v') - u' \mu_i(u')$, which is equal to $0$ as $i \notin I$. We thus have 
			\begin{equation}
				\label{eq:aux-6}
				\int_{u'}^{v'} \diff s \,  \theta_i^{(\varepsilon)}  \left(\left(\bM\btheta_*^{(I)}\right)_i - \left(\bM\btheta^{(\varepsilon)}\right)_i\right) \xrightarrow[\varepsilon\to 0]{} 0 \, .
			\end{equation}
			\item Further, if $s = u'$ or $s = v'$, 
			\begin{align}
				\label{eq:aux-4}
				\begin{split}
					V(\btheta^{(\varepsilon)}(s)) &= \sum_{i\in I} \left(\theta_i^{(\varepsilon)}(s) - \theta_{*,i}^{(I)} \log \theta_i^{(\varepsilon)}(s)\right) \\
					&= \sum_{i\in I} \theta_i^{(\varepsilon)}(s) + \left(\log \frac{1}{\varepsilon}\right)\sum_{i\in I} \theta_{*,i}^{(I)} w_i^{(\varepsilon)}(s) 
				\end{split}
			\end{align}
			by the definition of $w_i^{(\varepsilon)}$ in Equation \eqref{eq:def-w}. The first term $\sum_{i\in I} \theta_i^{(\varepsilon)}(s)$ is bounded independently of $\varepsilon$ by Lemma \ref{lem:bounded}. Further, for all $i \in I$, $\mu_i(s) > 0$ and thus by Equation \eqref{eq:identification-LCPs}, $z_i(s) > 0$. By the complementary slackness of $\bz(s)$ and $\bw(s)$, we must have $w_i(s) = 0$. As a consequence, using Proposition \ref{prop:uniform-convergences},
			\begin{align*}
				\sum_{i\in I} \theta_{*,i}^{(I)} w_i^{(\varepsilon)}(s) \xrightarrow[\varepsilon \to 0]{} \sum_{i\in I} \theta_{*,i}^{(I)} w_i(s) = 0 \, .
			\end{align*} 
			Returning to Equation \eqref{eq:aux-4}, we obtain that
			\begin{align}
				\label{eq:aux-7}
				&V(\btheta^{(\varepsilon)}(u')) = o\left(\log \frac{1}{\varepsilon}\right) \, , &&V(\btheta^{(\varepsilon)}(v')) = o\left(\log \frac{1}{\varepsilon}\right) \, .
			\end{align} 
		\end{itemize}
		Putting together Equations \eqref{eq:aux-5}, \eqref{eq:aux-6} and \eqref{eq:aux-7}, we obtain 
		\begin{align}
			\label{eq:aux-9}
			\int_{u'}^{v'} \diff s \Vert \btheta^{(\varepsilon)} - \btheta_*^{(I)} \Vert^2 \xrightarrow[\varepsilon \to 0]{} 0 \, .
		\end{align}
		To conclude on uniform convergence, we use an elementary argument based on monotonicity:
		\begin{align}
			\label{eq:aux-8}
			\max_{s \in [u,v]} \left\vert \theta_i^{(\varepsilon)}(s) -  \theta_{*,i}^{(I)} \right\vert = \max_{s \in [u,v]} \max\left(\theta_i^{(\varepsilon)}(s) -  \theta_{*,i}^{(I)}, \theta_{*,i}^{(I)} - \theta_i^{(\varepsilon)}(s)\right) \, .
		\end{align}
		By Lemma \ref{lem:nondecreasing}, for all $s \in [u,v]$ and $s' \in [v,v']$, $\theta_i^{(\varepsilon)}(s) -  \theta_{*,i}^{(I)} \leq \theta_i^{(\varepsilon)}(s') -  \theta_{*,i}^{(I)}$, thus
		\begin{align*}
			\theta_i^{(\varepsilon)}(s) -  \theta_{*,i}^{(I)} \leq \frac{1}{v'-v} \int_v^{v'} \diff s' \left(\theta_i^{(\varepsilon)}(s') -  \theta_{*,i}^{(I)}\right) \, .
		\end{align*}
		Using H\"older's inequality, we obtain 
		\begin{align*}
			\theta_i^{(\varepsilon)}(s) -  \theta_{*,i}^{(I)} &\leq \frac{1}{\sqrt{v'-v}} \left(\int_v^{v'} \diff s' \left(\theta_i^{(\varepsilon)}(s') -  \theta_{*,i}^{(I)}\right)^2\right)^{1/2} \\
			&\leq \frac{1}{\sqrt{v'-v}} \left(\int_v^{v'} \diff s' \left\Vert\btheta^{(\varepsilon)}(s') -  \btheta_{*}^{(I)}\right\Vert^2\right)^{1/2} \, .
		\end{align*}
		Similarly, 
		\begin{align*}
			\theta_{*,i}^{(I)} - \theta_i^{(\varepsilon)}(s)
			&\leq \frac{1}{\sqrt{u-u'}} \left(\int_{u'}^{u} \diff s' \left\Vert\btheta^{(\varepsilon)}(s') -  \btheta_{*}^{(I)}\right\Vert^2\right)^{1/2} \, .
		\end{align*}
		Finally, plugging these estimates back in Equation \eqref{eq:aux-8} and using Equation \eqref{eq:aux-9}, we obtain
		\begin{align*}
			\max_{s \in [u,v]} \left\vert \theta_i^{(\varepsilon)}(s) -  \theta_{*,i}^{(I)} \right\vert &\leq \max\Biggl(\frac{1}{\sqrt{v'-v}} \left(\int_v^{v'} \diff s' \left\Vert\btheta^{(\varepsilon)}(s') -  \btheta_{*}^{(I)}\right\Vert^2\right)^{1/2}, \\
			&\qquad\qquad\frac{1}{\sqrt{u-u'}} \left(\int_{u'}^{u} \diff s' \left\Vert\btheta^{(\varepsilon)}(s') -  \btheta_{*}^{(I)}\right\Vert^2\right)^{1/2}\Biggl) \\
			&\xrightarrow[\varepsilon\to 0]{} 0 \, .
		\end{align*}
		This being true for all $i \in \{1, \dots, d\}$, we conclude that $\btheta^{(\varepsilon)}$ converges to $\btheta_*^{(I)}$ uniformly on $[u,v]$ and thus point \eqref{it:main-2} of the theorem holds. 
	\end{enumerate}
\end{proof}

\section{Proof of Corollary \ref{coro:main}}
\label{sec:proof-coro-main}

We first study under which condition on $s$ we have $I(s) = \{1,\dots,d\}$. By definition of $I(s)$, this is equivalent to having $\bmu(s) > \bzero$ and thus $I(s) = \{1,\dots,d\}$ if and only if $f(\btheta) + \frac{1}{s}\langle \bk, \btheta \rangle$ is minimized at a positive point. Note that 
\begin{equation*}
	f(\btheta) + \frac{1}{s}\langle \bk, \btheta \rangle  = \frac{1}{2} \Vert \by \Vert^2 - \langle \br - \frac{1}{s} \bk, \btheta \rangle + \frac{1}{2} \langle \btheta,  \bM \btheta \rangle 
\end{equation*}
is minimized at $\bM^{-1}\left(\br - \frac{1}{s} \bk\right)$ thus 
\begin{align}
	\label{eq:I(s)-full}
	\begin{split}
		I(s) = \{1,\dots,d\} \quad &\Leftrightarrow \quad \bM^{-1}\br - \frac{1}{s} \bM^{-1}\bk > \bzero \\ &\Leftrightarrow \quad s > s_* \, , \qquad s_* := \max_i \frac{(\bM^{-1}\bk)_i}{(\bM^{-1}\br)_i} \, .
	\end{split}
\end{align}
Consider $s > s_*$. From Theorem \ref{thm:main}.\eqref{it:main-2},
\begin{equation*}
	\btheta^{(\varepsilon)}\left(s \log \frac{1}{\varepsilon}\right) \xrightarrow[\varepsilon \to 0]{} \btheta_*^{(\{1,\dots,d\})} \, .
\end{equation*}
Take $\eta > 0$. Then there exists $\varepsilon_1 > 0$ such that for all $\varepsilon \leq \varepsilon_1$, 
\begin{equation*}
	\left\Vert \btheta^{(\varepsilon)}\left(s \log \frac{1}{\varepsilon}\right) -  \btheta_*^{(\{1,\dots,d\})}  \right\Vert \leq \eta \, .
\end{equation*}
Thus $\tau_\eta^{(\varepsilon)} \leq s \log \frac{1}{\varepsilon}$. This being true for all $\varepsilon \leq \varepsilon_1$, we have $\limsup_{\varepsilon\to 0} \frac{\tau_{\eta}^{(\varepsilon)}}{\log 1/\varepsilon} \leq s$. This being true for all $s > s_*$, we have $\limsup_{\varepsilon\to 0} \frac{\tau_{\eta}^{(\varepsilon)}}{\log 1/\varepsilon} \leq s_*$.

We are left with showing that $\liminf_{\varepsilon\to 0} \frac{\tau_{\eta}^{(\varepsilon)}}{\log 1/\varepsilon} \geq s_*$. We assume that $\eta < \min_j \theta_{*,j}^{(\{1,\dots,d\})}$, which is possible as from Proposition \ref{prop:fixed-points}, $\btheta_{*}^{(\{1,\dots,d\})} > \bzero$. Consider $s < s_*$. Then from \eqref{eq:I(s)-full}, $I(s) \neq \{1,\dots,d\}$ thus we can consider $i \notin I(s)$. Assume further that $s \neq s_1, \dots, s_q$. Then by Theorem \ref{thm:main}.\eqref{it:main-2}, 
\begin{equation*}
	\theta^{(\varepsilon)}_i\left(s \log \frac{1}{\varepsilon}\right) \xrightarrow[\varepsilon \to 0]{} \theta_{*,i}^{(I(s))} = 0 \, .
\end{equation*}
Denote $\nu = \min_j \theta_{*,j}^{(\{1,\dots,d\})} - \eta > 0$. There exists $\varepsilon_2 > 0$ such that for all $\varepsilon\leq \varepsilon_2$, 
\begin{equation}
	\label{eq:aux-10}
	\theta^{(\varepsilon)}_i\left(s \log \frac{1}{\varepsilon}\right) < \nu \, .
\end{equation}
Define $\varepsilon_0$ as in Lemma \ref{lem:nondecreasing} and assume $\varepsilon \leq \min(\varepsilon_0, \varepsilon_2)$ so that both Lemma \ref{lem:nondecreasing} and Equation~\eqref{eq:aux-10} apply. Then for all $t \leq s \log \frac{1}{\varepsilon}$, 
\begin{align*}
	\theta^{(\varepsilon)}_i(t) \underset{\text{(Lemma \ref{lem:nondecreasing})}}{\leq } \theta^{(\varepsilon)}_i\left(s \log \frac{1}{\varepsilon}\right) \underset{\text{(Equation \eqref{eq:aux-10})}}{< } \nu  \, .
\end{align*}
Thus 
\begin{align*}
	\left\Vert \btheta^{(\varepsilon)}(t) - \btheta_{*}^{(\{1,\dots,d\})}\right\Vert &\geq \left\vert\theta_i^{(\varepsilon)}(t) - \theta_{*,i}^{(\{1,\dots,d\})} \right\vert \geq  \theta_{*,i}^{(\{1,\dots,d\})} - \theta_i^{(\varepsilon)}(t) \\
	&> \theta_{*,i}^{(\{1,\dots,d\})} - \nu \geq \min_j \theta_{*,j}^{(\{1,\dots,d\})} - \nu = \eta \, .
\end{align*}
This being true for all $t \leq s \log \frac{1}{\varepsilon}$, this gives $\tau_{\eta}^{(\varepsilon)} \geq s \log \frac{1}{\varepsilon}$. This being true for all $\varepsilon \leq \min(\varepsilon_0, \varepsilon_2)$, this gives $\liminf_{\varepsilon\to 0} \frac{\tau_{\eta}^{(\varepsilon)}}{\log 1/\varepsilon} \geq s$. This being true for all $s < s_*$, $s \neq s_1, \dots, s_q$, this gives $\liminf_{\varepsilon\to 0} \frac{\tau_{\eta}^{(\varepsilon)}}{\log 1/\varepsilon} \geq s_*$.

We thus conclude that $\lim_{\varepsilon\to 0} \frac{\tau_{\eta}^{(\varepsilon)}}{\log 1/\varepsilon} = s_*$.

\section{Conclusion}

In this paper, we have shown how the implicit regularization of DLNs is generated by incremental learning with successive coordinate activations. We obtain a sharp description of the incremental learning process using an associated regularized optimization problem with decreasing regularization. 

An immediate open question is to obtain a similar description without the anti-correlation assumption \ref{it:ass-M}. This would cover the overparametrized setting. In this case, it should be necessary to parametrize $\theta_i = (u_i^2 - v_i^2)/4$ (as in the article of \citet{vaskevicius2019implicit}, for instance) so that the sign of $\theta_i$ is not constrained. 

Further, we leave open whether our strategy can be adapted to study incremental learning in matrix factorization problems and more general neural networks, as well as the statistical benefits of the induced implicit regularization.

\section*{Acknowledgements}

The author thanks Loucas Pillaud-Vivien for many discussions motivating this project.



\newpage

\appendix
\section{Properties of the Linear Complementarity Problem}
\label{sec:properties}

This section gathers a few properties of the linear complementarity problem (LCP) from the monograph of~\citet{cottle2009linear}. Let $\bq \in \R^d$ and $\bM \in \R^{d \times d}$. We recall that the LCP associated to the parameters $\bq$ and $\bM$ is the problem of finding $(\bw,\bz) \in \left(\R^d\right)^2$ such that 
\begin{align}
	\label{eq:LCP}
	\begin{split}
		&\bw = \bq + \bM \bz \, , \\ 
		&\bw \geq \bzero \, , \qquad \bz \geq \bzero \, , \qquad \bw^\top \bz = 0 \, .
	\end{split}
\end{align}

\begin{proposition}
	\label{prop:unicity-sol-lcp}
	Assume that $\bM$ is symmetric positive definite. Then the LCP \eqref{eq:LCP} has a unique solution. 
\end{proposition}

This result is provided in the monograph of {\citet[Theorem 3.1.6]{cottle2009linear}} (actually without the symmetry requirement). 

\begin{proposition}[antitonicity property]
	\label{prop:antitonicity}
	Assume that $\bM$ is symmetric positive definite, with non-positive off-diagonal entries. Consider $\bq^{(1)} \leq \bq^{(2)}$ and let $(\bw_1,\bz_1)$, $(\bw_2,\bz_2)$ be the unique solutions of~\eqref{eq:LCP} with $\bq = \bq^{(1)}, \bq^{(2)}$ respectively. Then $\bz_1 \geq \bz_2$. 
\end{proposition}

\begin{proof}
	This result is provided in the monograph of \citet[Theorem 3.11.9]{cottle2009linear}. 
	
	Indeed, the fact that $\bM$ has non-positive off-diagonal entries means that $\bM$ is a $\bZ$-matrix in the sense of \citet[Definition 3.11.1]{cottle2009linear}. Further, $\bM$ is a symmetric positive definite matrix, thus $\bM$ is a $\bP$-matrix in the sense of \citet[Section 3.3]{cottle2009linear}. Thus $\bM$ is a $\bK$-matrix in the sense of \citet[Definition 3.11.1]{cottle2009linear}. Thus Theorem 3.11.9 of \citet{cottle2009linear} applies.
\end{proof}

\section{Proof of Lemma \ref{lem:nondecreasing}}
\label{sec:proof-prop-nondecreasing}

The DLN dynamics \eqref{eq:GF-reparametrized} form an autonomous ordinary differential equation (ODE) 
\begin{align}
	\label{eq:ODE}
	&\frac{\diff \btheta}{\diff t} = \Psi(\btheta) \, , &&\left(\Psi(\btheta)\right)_i = \theta_i \left(r_i - \sum_{j=1}^{d} M_{ij}\theta_j\right) \, .
\end{align} 
We first show that the set 
\begin{equation*}
	Q = \left\{\btheta \geq \bzero \, \middle\vert \, \forall i\in\{1,\dots,d\}, r_i - \sum_{j=1}^{d} M_{ij} \theta_j \geq 0 \right\}
\end{equation*}
is positively invariant for this ODE, i.e.,~if $\btheta(0) \in Q$, then $\btheta(t) \in Q$ for all $t \geq 0$. The proof is based on Nagumo's theorem, see the original result of \citet{nagumo1942lage} or the recent introduction of \citet[Section 4.2]{blanchini2008set} for instance. Heuristically, Nagumo's theorem states that $Q$ is positively invariant if the vector field $\Psi(\btheta)$ points ``in'' the set $Q$ if $\btheta$ is on the boundary of $Q$. 

More precisely, for $\btheta \in Q$, denote 
\begin{equation*}
	\text{Act}(\btheta) = \left\{i\in \{1,\dots,d\}\,\middle\vert \, r_i - \sum_j M_{ij} \theta_j = 0 \right\}
\end{equation*}
the set of active constraints at $\btheta$ and 
\begin{equation*}
	T_Q(\btheta) = \left\{\nu \in \R^d \, \middle\vert \, \forall i \in \text{Act}(\btheta), -\sum_jM_{ij}\nu_j \geq 0\right\} 
\end{equation*}
the tangent cone to $Q$ at $\btheta$ \cite[Eq. (4.6)]{blanchini2008set}. Then Nagumo's theorem \cite[Corollary 4.8]{blanchini2008set} states that $Q$ is positively invariant for the dynamics \eqref{eq:ODE} if for all $\btheta\in Q$, $\Psi(\btheta) \in T_Q(\btheta)$.

We now check that this latter condition is satisfied. Let $\btheta \in Q$ and $i \in \text{Act}(\btheta)$. We need to show that $-\sum_j M_{ij} \left(\Psi(\btheta)\right)_j \geq 0$. We have
\begin{align*}
	-\sum_j M_{ij} \left(\Psi(\btheta)\right)_j &= -\sum_j M_{ij} \theta_j \left(r_j - \sum_k M_{jk} \theta_k\right)   \\
	&= -M_{ii} \theta_i \left(r_i - \sum_k M_{ik} \theta_k\right) -\sum_{j \neq i} M_{ij} \theta_j \left(r_j - \sum_k M_{jk} \theta_k\right) \, . 
\end{align*}
For the first term, we have $i \in \text{Act}(\btheta)$ and thus $r_i - \sum_k M_{ik} \theta_k = 0$. For the sum, we have $M_{ij} \leq 0$ (Assumption \ref{it:ass-M}), $\theta_j \geq 0$ and $r_j - \sum_k M_{jk} \theta_k \geq 0$ (because $\btheta \in Q$). Thus we indeed have $-\sum_j M_{ij} \left(\Psi(\btheta)\right)_j \geq 0$ and thus $\Psi(\btheta) \in T_Q(\btheta)$. We conclude that $Q$ is positively invariant.

As $\btheta^{(\varepsilon)}(0) = (C_1 \varepsilon^{k_1} , \dots , C_d \varepsilon^{k_d}) \to \bzero$ as $\varepsilon\to 0$ and $\br > 0$, there exists $\varepsilon_0 > 0$ such that $\forall \varepsilon \in (0,\varepsilon_0], \btheta^{(\varepsilon)}(0) \in Q$. Thus $\forall \varepsilon \in (0,\varepsilon_0], \forall t \geq 0, \btheta^{(\varepsilon)}(t) \in Q$. Thus $\forall \varepsilon \in (0,\varepsilon_0], \forall t \geq 0, \forall i \in \{1,\dots,d\}$,
\begin{equation*}
	\frac{\diff \theta_i}{\diff t}  = \theta_i \left( r_i - \sum_{j=1}^{d} M_{ij} \theta_j \right) \geq 0\, .
\end{equation*}

\section{Proof of Proposition \ref{prop:positive-definite}}
\label{sec:proof-prop-positive-definite}

Consider the block matrix
\renewcommand{\arraystretch}{1.5}
\begin{equation*}
	\widetilde{\bM} = 
	\left(
	\begin{array}{c|c}
		\bX & - \by \\ 
	\end{array}
	\right)^\top 
	\left(
	\begin{array}{c|c}
		\bX & - \by \\ 
	\end{array}
	\right)
	= \left(
	\begin{array}{c|c}
		\bM & - \br \\ \hline
		-\br^\top  & \Vert \by \Vert^2 \\
	\end{array}
	\right) \, .
\end{equation*}
From Assumptions \ref{it:ass-r}-\ref{it:ass-M}, $\widetilde{\bM}$ is a matrix with non-positive off-diagonal entries. Thus there exists $\mu \in \R$ such that $\bA = \mu \bI_{d+1} - \widetilde{\bM}$ is a matrix with non-negative entries. Moreover, from Assumption \ref{it:ass-r}, for $i\in\{1,\dots,d\}$, $A_{i,d+1} = A_{d+1,i} = r_i > 0$. This implies that $\bA$ is irreducible (see \citealp[Section 2.2]{cottle2009linear} for a definition). By the Perron-Frobenius theorem \cite[Theorem 2.2.21]{cottle2009linear}, the largest eigenvalue of $\bA$ is simple and there exists an eigenvector $\widetilde{\bv}$ with positive entries associated to this eigenvalue. As a consequence, the smallest eigenvalue $\lambda$ of $\widetilde{\bM} = \mu \bI_{d+1}-\bA $ is simple and associated to $\widetilde{\bv}$. 

We now have two cases:
\begin{itemize}
	\item If the smallest eigenvalue $\lambda$ is positive, then $\widetilde{\bM}$ is positive definite and thus so is the principal submatrix $\bM$.
	\item If $\lambda = 0$, then $\ker \widetilde{\bM} = \{\alpha \widetilde{\bv}, \alpha \in \R\}$. We want to show that $\ker \bM = \{\bzero\}$. Consider $\bv \in \R^d$ such that $\bM \bv = \bX^\top \bX \bv = \bzero$. This implies that $\bX \bv = \bzero$. (This can be seen, for instance, using a singular value decomposition of $\bX$.) Then we perform the block computation 
	\begin{align*}
		\widetilde{\bM} \left(
		\begin{array}{c}
			\bv \\ \hline
			0 \\
		\end{array}
		\right)
		&= \left(
		\begin{array}{c|c}
			\bM & - \br \\ \hline
			-\br^\top  & \Vert \by \Vert^2 \\
		\end{array}
		\right) 
		\left(
		\begin{array}{c}
			\bv \\ \hline
			0 \\
		\end{array}
		\right)  = 	
		\left(
		\begin{array}{c}
			\bM \bv \\ \hline
			-\br^\top \bv \\
		\end{array}
		\right) \\
		&= 	
		\left(
		\begin{array}{c}
			\bM \bv \\ \hline
			-\by^\top \bX \bv \\
		\end{array}
		\right) = \bzero \, .
	\end{align*}
	Thus $\left(
	\begin{array}{c}
		\bv \\ \hline
		0 \\
	\end{array}
	\right) \in \ker \widetilde{\bM} = \{\alpha \widetilde{\bv}, \alpha \in \R\}$. As $\widetilde{\bv} > \bzero$, elements of $\ker \widetilde{\bM}$ have all entries non-zero or all entries equal to $0$. Thus it must be that $\bv = \bzero$. This concludes that $\ker \bM = \{\bzero\}$ and thus that $\bM$ is positive definite. 
\end{itemize}

\section{Proof of Proposition \ref{prop:fixed-points}}
\label{sec:proof-prop-fixed}

Let $I \subset \{1, \dots, d\}$. 

We first prove that $(\bM_{II})^{-1}$ exists. The matrix $\bM_{II}$ has non-positive off-diagonal entries thus $\bM_{II}$ is a $\bZ$-matrix in the sense of \citet[Definition 3.11.1]{cottle2009linear}. Further, $\bM_{II}$ is a symmetric positive definite matrix as a principal submatrix of a positive definite matrix (Proposition \ref{prop:positive-definite}). Thus $\bM_{II}$ is a $\bP$-matrix in the sense of \citet[Section 3.3]{cottle2009linear}. Thus $\bM_{II}$ is a $\bK$-matrix in the sense of \citet[Definition 3.11.1]{cottle2009linear}. From Theorem 3.11.10 of \citet{cottle2009linear}, $(\bM_{II})^{-1}$ exists and has non-negative entries. 

Thus, we can define $\btheta_*^{(I)} \in \R^d$ by the equations $(\btheta_*^{(I)})_I = (\bM_{II})^{-1} \br_I$ and $(\btheta_*^{(I)})_{I^c} = \bzero$. 

We check that $(\btheta_*^{(I)})_I > 0$. Fix $i \in I$. Then $\theta_{*,i}^{(I)} = \sum_{j\in I} \left((\bM_{II})^{-1}\right)_{ij} r_j \geq 0$ from Assumption \ref{it:ass-r} and the fact that $(\bM_{II})^{-1}$ has non-negative entries. Moreover, assume by contradiction that $\theta_{*,i}^{(I)} = 0$. As from Assumption \ref{it:ass-r}, $\br > \bzero$, we have for all $j \in I$, $\left((\bM_{II})^{-1}\right)_{ij} = 0$. Thus a full row of $(\bM_{II})^{-1}$ is~$\bzero$, which contradicts the fact that $(\bM_{II})^{-1}$ is invertible. Thus for all $i \in I$, $\theta_{*,i}^{(I)} > 0$. 

We now check that $\btheta_*^{(I)}$ is a fixed point of \eqref{eq:GF-reparametrized}. Fix $i \in \{1, \dots, d\}$. If $i \in I$,
\begin{equation*}
	r_i - \sum_{j=1}^{d} M_{ij} \theta_{*,j}^{(I)} = r_i - \sum_{j \in I} M_{ij} \left((\bM_{II})^{-1} \br_I\right)_j = \left(\br_I - \bM_{II} (\bM_{II})^{-1} \br_I \right)_i = 0 \, .
\end{equation*}
If $i \notin I$, by definition, $\theta_{*,i}^{(I)} = 0$. In both cases, $\theta_{*,i}^{(I)} \left(r_i - \sum_{j=1}^{d} M_{ij} \theta_{*,j}^{(I)} \right) = 0$. As this is true for all $i \in \{1, \dots, d\}$, this proves that $\btheta_*^{(I)}$ is a fixed point of \eqref{eq:GF-reparametrized}. 

We now take a fixed point $\btheta$ of \eqref{eq:GF-reparametrized} with support $I$, and show that $\btheta = \btheta_*^{(I)}$. It is sufficient to show the equality on the support of the vectors, namely that $\btheta_I = \left(\btheta_*^{(I)}\right)_I = (\bM_{II})^{-1} \br_I$. Consider $i \in I$. As $\btheta$ is a fixed point, $\theta_i \left( r_i - \sum_{j=1}^{d} M_{ij} \theta_j \right) = 0$. But as $i \in I$, the first factor is non-zero. Thus $r_i - \sum_{j=1}^{d} M_{ij} \theta_j = r_i - \sum_{j \in I} M_{ij} \theta_j = 0$. With vector notation, we proved $\br_I - \bM_{II}\btheta_I = 0$, which gives the claimed equality.  

\section{Proof of Proposition \ref{prop:optim-LCP}}
\label{sec:proof-prop-optim-LCP}

In this proof, we use convex duality \cite[Section 5.5]{boyd2004convex}. The Lagrangian associated to \eqref{eq:optim} is 
\begin{equation*}
	L(\btheta, \bw) = \langle \bq, \btheta \rangle + \frac{1}{2} \langle \btheta, \bM \btheta \rangle - \langle \bw, \btheta \rangle
\end{equation*}
where $\bw \in \R^d$ is the Lagrange multiplier associated to the constraint $\btheta \geq \bzero$. As the optimization problem \eqref{eq:optim} is convex, the KKT conditions are necessary and sufficient for optimality. The stationarity condition is 
\begin{equation*}
	\bzero = \nabla_{\btheta} L(\btheta,\bw) = \bq + \bM \btheta - \bw \, ,
\end{equation*} 
the feasibility conditions are $\btheta \geq \bzero$ and $\bw \geq \bzero$, and the complementary slackness condition is $\bw^\top \btheta = 0$. At this point, we have proven the equivalence between:
\begin{enumerate}
	\item $\bz$ is a minimizer of the constrained optimization problem 
	\begin{equation*}
		\underset{\btheta \in \R^d, \, \btheta \geq \bzero}{\rm{min}.} \left\{ \langle \bq, \btheta \rangle + \frac{1}{2} \langle \btheta, \bM \btheta \rangle \right\} \, .
	\end{equation*}
	\item There exists $\bw \in \R^d$ such that $(\bw,\bz)$ is a solution of 
	\begin{align*}
		&\bw = \bq + \bM \bz \, , \\
		&\bw \geq \bzero \, , \qquad \bz \geq \bzero \, , \qquad \bw^\top \bz = 0 \, .
	\end{align*}
\end{enumerate}
We are left with proving that the solutions of both problems are indeed unique. For the LCP, this is given by Proposition \ref{prop:unicity-sol-lcp} as $\bM$ is positive definite (Proposition~\ref{prop:positive-definite}). We can then use the equivalence shown above to prove that the constrained optimization problem~\eqref{eq:optim} has a unique solution.

\vskip 0.2in
\bibliography{biblio}

\end{document}